\documentclass{ar-1col}
\usepackage[numbers]{natbib}
\usepackage{url}
\usepackage{amsmath} 
\usepackage{amssymb}  
\usepackage{graphicx}
\usepackage[linesnumbered,ruled,vlined,algo2e]{algorithm2e}
\usepackage{mathptmx} 
\usepackage{amsthm}

\usepackage{csquotes}
\usepackage{subcaption}

\newtheorem{assumption}{Assumption}
\newtheorem{definition}{Definition}
\newtheorem{lemma}{Lemma}
\newtheorem{theorem}{Theorem}

\newcommand{\bt}{\mathcal{T}}

\newtheorem{remark2}{Remark}

\jname{Annual Review of Control, Robotics, and Autonomous Systems}
\jvol{5}
\jyear{2022}
\doi{10.1146/annurev-control-042920-095314}

\begin{document}

\markboth{Ögren and Sprague}{Behavior Trees in Robot Control Systems}

\title{Behavior Trees in \\Robot Control Systems}

\author{Petter Ögren,$^1$
and Christopher I. Sprague$^1$
\affil{$^1$KTH Royal Institute of Technology, Stockholm, Sweden, SE-10044; email: petter,sprague@kth.se}
}

\begin{abstract}

In this paper we will give a control theoretic perspective on the research area of behavior trees in robotics.
The key idea underlying behavior trees is to make use of modularity, hierarchies and feedback, in order to handle the complexity of a versatile robot control system. 
Modularity is a well-known tool to handle software complexity by enabling development, debugging and extension of separate modules without having detailed knowledge of the entire system. A hierarchy of such modules is natural, since robot tasks can often be decomposed into a hierarchy of sub-tasks. Finally, feedback control is a fundamental tool for handling uncertainties and disturbances in any low level control system, but in order to enable feedback control on the higher level, where one module decides what submodule to execute,
information regarding progress and applicability of each submodule needs to be shared in the module interfaces.

We will describe how these three concepts come to use in
theoretical analysis, practical design, as well as extensions and combinations with other ideas from control theory and robotics.

\end{abstract}

\maketitle

\tableofcontents


\section{Introduction}
In this section we will describe why modularity, hierarchical structure, and feedback are useful in robot control systems, and how these three concepts are combined in a control structure called behavior trees (BTs).

The rapid development of robotic hardware and software has enabled the use of robots to 
expand beyond structured factory environments, into our homes, streets and diverse workplaces.
In these new settings, robots often need a wide range of capabilities, including the possibility to add even more features by online software updates.
It is well known that adding features to software increases complexity, which in turn increases the cost of development \cite{blumeHierarchicalModularity1999}. 
It is also well known that modularity is a key principle that can be used to reduce complexity. By dividing a system into modules with well defined interfaces and functionality, each module can be developed, tested, and extended without having detailed knowledge about the rest of the system.
Thus, there is reason to believe that modularity, in terms of well defined interfaces and functionality, is an important property also for a robot control system.

\begin{figure}[t]
	\centering
	\includegraphics[width=0.9\textwidth]{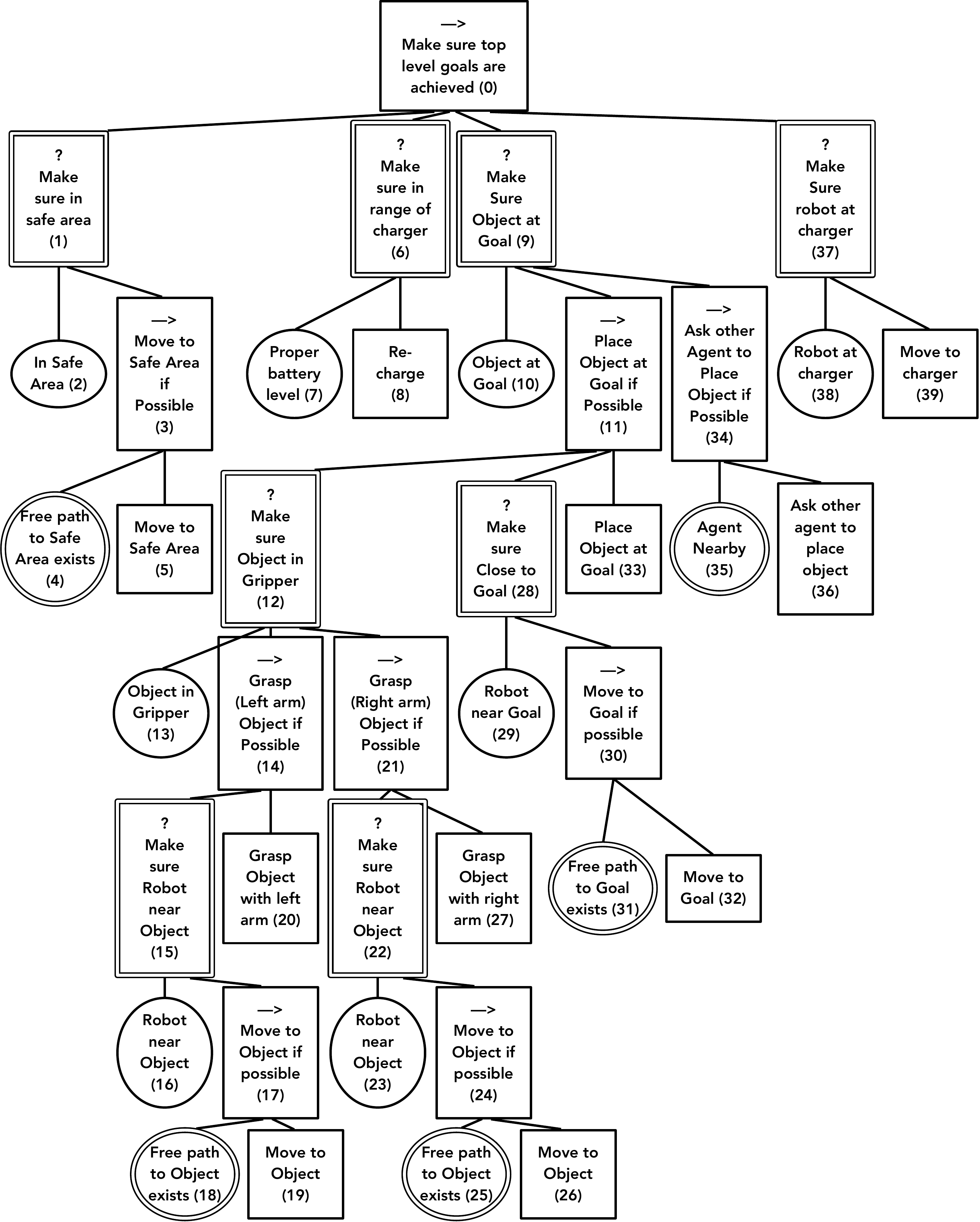}
	\caption{A mobile manipulator BT. The top level goals can be found in the top row: \emph{Make sure in safe area (1)}, \emph{Make sure in range of charger (6)}, \emph{Make sure object at goal (9)} and \emph{Make sure robot at  charger (37)}, in order of priority. If, at some time instant the action, \emph{Move to object (26)}, is executing, a human operator can easily understand why this action was chosen by reading every double stroked module on the branch towards the root: \emph{Move to object (26)}, in order to \emph{Make sure robot near object (22)}, (to) \emph{Make sure object in gripper (12)}, (to) \emph{Make sure object at goal(9)}. The  meaning of the double strokes will be explained in detail in Section~\ref{sec_design}.
	}
	\label{bt_big}
\end{figure}

A natural extension of modularity is hierarchical modularity
\cite{blumeHierarchicalModularity1999}, where modules can contain submodules and so on. The rationale for such a structure is the simple observation that when a system grows, a single layer of modules either results in a very large number of modules, or in modules that are themselves very large. Thus, the benefit of modularity is strengthened if the modules can contain submodules in a hierarchical fashion.
There is an additional reason why hierarchical modularity makes sense in robot control systems, and this is the fact that many robot tasks can naturally be divided into subtasks in a hierarchical way, an observation  that is underlying e.g. hierarchical task networks
\cite{sacerdotiStructurePlansBehavior1975,erolUMCPSoundComplete1994}. For example, fetching an item might involve moving to a cupboard and opening it, which might in turn involve grasping a handle, and so on.

To make the control system modular, we will make the actual control policy, the mapping from state to action, modular.
In many applications within robotics and control there is a need to compose a control policy out of a set of subpolicies.
In an autonomous car there might be subpolicies for parking, overtaking, lane keeping, handling intersections etc, and
in a mobile manipulator there might be subpolicies for grasping, docking with a recharger, moving from A to B etc.

Feedback is perhaps the most important principle in control theory, and the property that separates open loop control from closed loop control.
In open loop control, a series of commands are executed over time according to some form of plan, while in closed loop control, the issued commands are constantly adapted based on current information obtained from monitoring key parts of the world state. 
It is clear that classical closed loop control should be executed at the lowest level of a hierarchical modular robot control system, but it is less clear what kind of observations should be used between two hierarchical levels, to allow one module to use feedback when it determines what submodule to execute.
We will come back to this question shortly, but for now we just note that if a submodule fails with achieving its goal, we do not want the  parent module to just execute the next submodule in an open loop fashion, but instead chose the proper submodule using feedback, based on the fact that the previous one just failed.

BTs were created to combine feedback with a hierarchical modular design. Thus, modules should capture some functionality that can be combined into larger modules, with a clearly defined interface between modules on all levels. Furthermore, 
feedback regarding the execution should be passed up the module hierarchy using the same interface.

The formal definition of BTs can be found in Section~\ref{sec_formal}, but here we will make an informal description. We let each module in the discussion above be a BT. Thus a complex BT can contain a number of sub-BTs and so on, as illustrated in Figure~\ref{bt_big}, where each node in the graph is the root of a sub-BT. The interface of all BTs (the lines in the figure) is given in terms of a function call with return values. When a BT is called it returns two things, first the suggested control action and second the information needed to apply feedback control and determine what sub-BT to execute. This information, or metadata, regarding the execution and applicability of a module is given in terms of one out of three symbols, S (success), F (failure), and R (running). Thus, if a submodule for grasping a cup returns success, the next submodule in the intended sequence, such as lifting the cup,  might be invoked. If, on the other hand, the submodule returns failure, some kind of fallback action needs to be invoked, such as trying to re-grasp the cup, or getting a better sensor reading of its pose. Finally, if the submodule returns running it might be preferable to let the execution run for a while longer.


At this point we note that there are basically three fundamental reasons for stopping what you are doing and starting a new activity. Either you succeed and go on to the next action, or you fail, and need to handle this more or less unexpected fact, or an external event happened that makes the current action inappropriate. Imagine a robot tasked to fetch an object, such as in Figure~\ref{bt_big}.
If the robot is grasping the object, it might switch to moving if the grasping succeeds. If the grasping fails, it might try with the other arm.
A number of events might also occur to end the process of grasping the object. Another agent might put the object in its proper place (positive surprise, we are done), or, another agent might move the object further away (negative surprise, we need to move closer again), or the fire alarm might go off making the entire building unsafe (unrelated surprise, we need to leave the building).

The outline of this paper is as follows. 
In Section~\ref{sec_history} we give a brief history of BTs, followed by a formal definition in Section~\ref{sec_formal}.
Then, we investigate the property of modularity in some more detail in Section~\ref{sec_modularity}.
The issue of convergence is analyzed in Section~\ref{sec_convergence} followed by a design principle in Section~
\ref{sec_design}.
Safety guarantees and the connection to control barrier functions are treated in Section~\ref{sec_cbf}.
Then, we see how BTs are related to explainable AI in Section~\ref{sec_exai}, and can be connected to
reinforcement learning in Section~\ref{sec_rl}, evolutionary algorithms in Section~\ref{sec_ga} and planning in Section~\ref{sec_planning}.
Finally, conclusions,  together with a set of summary points and future important issues can be found in Section~\ref{sec_conclusions}.


\section{The history of behavior trees and their relation to finite state machines}
\label{sec_history}
The need for a modular hierarchical control structures is shared between the domains of robotics and computer games.
However, low-level capabilities such as grasping and navigation are research areas in their own right in robotics, but trivial in the virtual worlds of a computer game. Therefore, computer game programmers started putting together larger sets of low level capabilities earlier than roboticists, and hence experienced the drawback of finite state machines (FSMs) described
below
 earlier as well.
BTs were thus proposed as a response to those drawbacks by programmers in the gaming industry. 
It is hard to determine who was the first to concieve BTs, as
important ideas were shared in partially documented workshops, conferences and blog posts.
However, important milestones were definitely passed through the work of Michael Mateas and Andrew Stern \cite{mateas_behavior_2002}, and Damian Isla \cite{islaHandlingComplexityHalo2005}.
The development continued in the game AI community, and a few years later the first journal paper on BTs appeared \cite{florez-puga_query-enabled_2009}, followed by the first papers on BTs in robotics, 
 independently described in \cite{ogrenIncreasingModularityUAV2012} and \cite{bagnell_integrated_2012}.
 Note that there is also a completely different tool called behavior trees, that is used for requirement analysis\footnote{A different concept with the same name: \url{https://en.wikipedia.org/wiki/Behavior_tree}}.   

As mentioned above, BTs were partially developed to improve modularity of (FSM) controllers.
FSMs, and in particular hierarchical FSMs  \cite{harelStatechartsVisualFormalism1987}, do have mechanisms for hierarchical modularity.
However, a key problem is that the transitions of a FSM are encoded inside the modules (states), thus each module needs to know about the existence and capabilities of the other modules, as well as the purpose of its own supermodule. In this way, each transition creates a dependence between two modules, and with $N$ modules there is $N^2$ possible transitions/dependencies. In comparison, a BT module only has to know if it succeeded or not.
Regarding expressivity,
it was shown in 
\cite{biggarExpressivenessHierarchyBehavior2021}  that BTs with internal variables are equally expressive as FSMs.
Thus, like with two general purpose programming languages, the choice between the two is not governed by what is possible, but rather what makes the design process smooth. A more detailed description of the relationship between BTs and FSM, as well as a broad overview of research on BTs can be found in the recent survey \cite{iovinoSurveyBehaviorTrees2020} and the book \cite{colledanchiseBehaviorTreesRobotics2018}.

\section{Definition of Behavior Trees}
\label{sec_formal}

In this section we will formally define BTs and their execution for both discrete and continuous-time systems. This formulation is based on
\cite{colledanchiseHowBehaviorTrees2017,ogrenConvergenceAnalysisHybrid2020,sprague2021ctrl} and chosen to enable a control-theoretic analysis of BTs\footnote{Other BT formulations exist, including  memory versions of interior nodes, and leaf nodes encapsulating the execution of the system dynamics, thereby allowing the parallel execution of two leaves e.g., controlling different motors on the same robot, see \cite{colledanchiseHowBehaviorTrees2017,colledanchiseImplementationBehaviorTrees2021}.}.

The core idea is to formally define a BT as a combination of a controller and a metadata function used to provide feedback regarding the execution. Using the metadata, these BTs can then be combined in order to create more complex BTs in a hierarchical tree structure, as in Figure~\ref{bt_big}, hence the name behavior tree.

Let the system state be $x \in X  \subset \mathbb{R}^n$ and the system dynamics be given by $\dot x = f(x,u)$ or $x_{t+1}=f(x_t,u)$, see Definition~\ref{def_execution}, with $u \in U \subset \mathbb{R}^k$. 

\begin{definition}[Behavior tree]
\label{def_bt}
 A BT $\bt_i$ is a pair
 
\begin{equation}
 \bt_i = (u_i, r_i)
\end{equation}
where $i$ is an index, $u_i:X \rightarrow U$ is the controller that runs when the BT is executing, and $r_i:X \rightarrow \{\mathcal{R},\mathcal{S},\mathcal{F} \}$ provides metadata regarding the applicability and progress of the execution.
\end{definition}

A BT can either be created through a hierarchical combination of other BTs, using the Sequence and Fallback operators described below, or it can be defined by directly specifying $u_i(x)$ and $r_i(x)$.

The metadata $r_i$ is interpreted as follows: 
\emph{Running} ($\mathcal{R}$),
\emph{Success} ($\mathcal{S}$), and
\emph{Failure} ($\mathcal{F}$).
Let the Running region ($R_i$),
Success region ($S_i$) and
Failure region ($F_i$) correspond to a partitioning\footnote{Throughout the paper we use the word partition, even though some of the sets might be empty.} of the state space,  defined as follows:
\begin{equation}
 R_i=\{x \in X: r_i(x)=\mathcal{R} \},~
 S_i=\{x\in X: r_i(x)=\mathcal{S} \},~
 F_i=\{x\in X: r_i(x)=\mathcal{F} \}. \nonumber
\end{equation}

\begin{definition}
\label{def_execution}
Assuming the BT $\bt_i$ is the root, and not a subtree of another BT, and $x\in R_i$,
the system evolves according to $\dot x = f(x,u_i(x))$ or $x_{t+1}=f(x_t,u_i(x))$ depending on if the system is continuous time or discrete time.

\end{definition}

\begin{remark2}
 If   $x\in S_i \cup F_i$ of the root BT, it has either succeeded or failed, and it is up to the user to apply an appropriate action, such as shutting down the robot or entering an idle mode. 
 If this is not desired, an additional top layer of the BT can be designed that executes
when the main BT returns success or failure. If this top layer always return running, we have  $S_i=F_i=\emptyset$ for the overall tree.
  \end{remark2}

The execution of a BT can thus be seen as a discontinuous dynamical system \cite{cortesDiscontinuousDynamicalSystems2008}, as illustrated in  Figure~\ref{operating_regions_0}.
Below, in Lemma \ref{lem_omega_executes}, we will show that if the BT $\bt_0$ is composed of a set of subtrees $\{\bt_i\}$ then the state space $X$ is divided into different so-called operating regions $\Omega_i$ such that $u_0(x)=u_i(x)$ if $x \in \Omega_i$,
as illustrated in Figure~\ref{operating_regions_0}.
Thus, in the continuous time case above, the execution will in most cases be a discontinuous dynamical system with corresponding issues regarding existence and uniqueness \cite{filippovDifferentialEquationsDiscontinuous1988}. Going deeper into these issues is beyond the scope of this paper, therefore we will just make the following assumption.

\begin{assumption}
 The BTs $\bt_i$ are defined in such a way that the execution in Definition \ref{def_execution} has solutions that exist and are unique.
\end{assumption}

\begin{figure}[t]
	\centering
	\includegraphics[width=4cm]{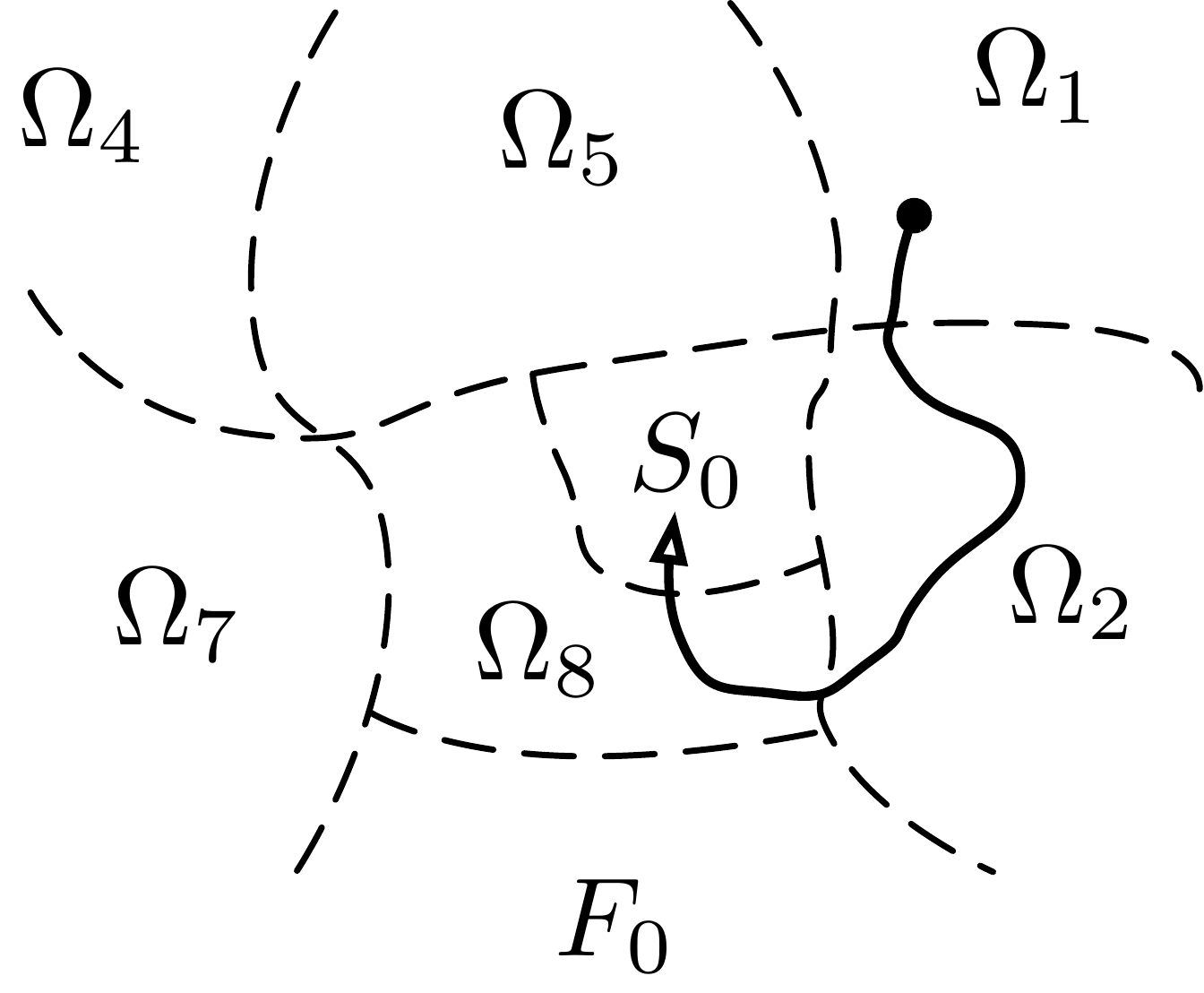}
	\caption{The state space $X$ is partitioned into a set of operating regions $\Omega_i$, and the global success region $S_0$ and failure region $F_0$. We want to design the BT such that the state reaches  $S_0$ and avoids $F_0$. The solid line illustrates an execution starting in $\Omega_1$ and ending in $S_0$.
	Note that the $\Omega_i$ of each subtree is not defined in the subtree itself, but depends on feedback in the form of the return status $r_i(x)$ of a number of neighboring subtrees as described in Definition~\ref{def_operating_region}. }
	\label{operating_regions_0}
\end{figure}

 As described above, the main point of BTs is to enable the creation of complex controllers from simpler ones in a modular fashion. There are two ways of combining BTs, the sequence and fallback compositions.

\begin{definition}\emph{(Sequence Compositions of BTs)}
\label{def_sequence}
 Two or more BTs  can be composed into a more complex BT using a Sequence operator,
 $\bt_0=\mbox{Sequence}(\bt_1,\bt_2).$ 
 Then $r_0,u_0$ are defined as follows
\begin{eqnarray}
   \mbox{If }x\in S_1:& 
   &r_0(x) =  r_2(x),~ 
   u_0(x) =  u_2(x) \label{bts:eq:seq1}\\ 
   \mbox{ else}:& 
   &r_0(x) =  r_1(x),~
   u_0(x) =  u_1(x) \label{bts:eq:seq2}
 \end{eqnarray}
\end{definition}
$\bt_1$ and $\bt_2$ are called children of $\bt_0$. Note that when executing $\bt_0$, the first child $\bt_1$ in (\ref{bts:eq:seq2}) is executed as long as it returns \emph{Running} or \emph{Failure} $(x_k \not \in S_1)$.  
The second child of the Sequence is executed in (\ref{bts:eq:seq1}), only when the first returns \emph{Success} $(x_k \in S_1)$. Finally, the Sequence itself,  $\bt_0$ returns \emph{Success} only when all children have succeeded $(x \in S_1 \cap S_2)$.

For notational convenience, we write
\begin{equation}
\mbox{Sequence}(\bt_1, \mbox{Sequence}(\bt_2,\bt_3))= \mbox{Sequence}(\bt_1,\bt_2, \bt_3),
\end{equation}
and similarly for arbitrarily long compositions. The sequence node is also denoted by ($\rightarrow$), as seen in Figure~\ref{bt_big}.

\begin{remark2}[Giving names to sequence and fallback nodes]
When drawing BTs, as in  Figure~\ref{bt_big}, the symbols $\rightarrow$ and $?$ are used to denote sequences and fallbacks. However, some users prefer to also give descriptive names to the subtrees starting from each node, to improve readability. We believe this is a useful practice, similar to choosing good names for functions when programming. Giving all nodes names improves readability, underlines the fact that all subtrees, including single leaf nodes, have an identical interface to its parent, see Definition~\ref{def_bt}, and is very convenient in combination with software GUIs that enable a subtree to be visually collapsed into a single node and expanded back again.
\end{remark2}

The advantage of properly named subtrees can be seen in Figure~\ref{bt_big}. As described in the caption, the reason for executing a leaf node is clear from the names of the subtrees it belongs to. This is discussed further in Section~\ref{sec_exai}, on explainable AI.


A key element of BTs is how the operating regions $\Omega_i$ of  Figure~\ref{operating_regions_0} depend on the success, failure and running regions, $S_i, F_i, R_i$ of all subtrees across a hierarchical structure. Thus we need to determine a number of properties of these sets.

\begin{lemma}
 If $\bt_0=\mbox{Sequence}(\bt_1,\bt_2)$, then Definition~\ref{def_sequence}  implies that
 \begin{align}
 S_0 &= S_1 \cap S_2,  \label{eq_seq_s}\\
 F_0 &= F_1 \bigcup (S_1 \cap F_2),\label{eq_seq_f}\\
 R_0 &= R_1 \bigcup (S_1 \cap R_2),
 \end{align}
\end{lemma}
\begin{proof}
 A straightforward application of  the definition gives the result above.
\end{proof}

\begin{figure}[t]
	\centering
	\includegraphics[height=4cm]{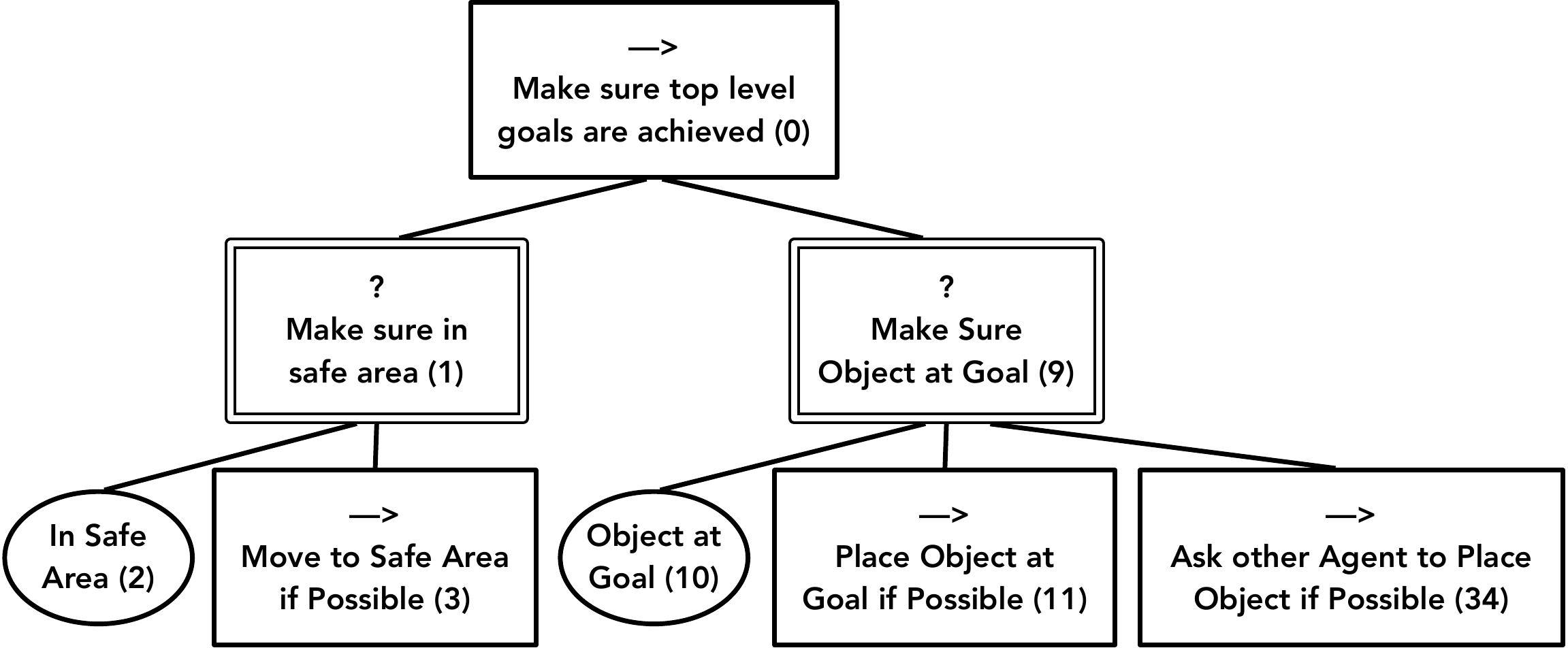}
	\caption{This mobile manipulator BT is a subset of the one in Figure~\ref{bt_big}, with the overall objective of moving a given object to a designated goal area, while staying in the safe part of the working area. All nodes are numbered with the index of the corresponding subtree. The sequence nodes 0, 3, 11, 34 are denoted by a $(\rightarrow)$ symbol and the fallback nodes 1, 9 by a $(?)$. Conditions are indicated by ovals. Note that the children of nodes 3, 11, 34 are not shown in this figure.}
	\label{seq_fall}
\end{figure} 
 
 Consider the Mobile Manipulator example in Figure~\ref{seq_fall}, which is actually a subset of Figure~\ref{bt_big}, where we have removed subtrees 6 and 37, and collapsed subtrees 3, 11 and 34 into single nodes. This was done to illustrate how the modularity enables analysis to be done at different levels. The root node, 0, is a sequence composition of subtrees 1 and 9. Equation (\ref{bts:eq:seq1}) now states that \emph{Make sure top level goals are achieved (0)} executes \emph{Make sure object at goal (9)}, $u_0(x)=u_9(x)$, only when \emph{Make sure in safe area (0)} returns success, $x\in S_1$. If that is not the case, node 1 will be executed, $u_0(x)=u_1(x)$. Similarly, Equation (\ref{eq_seq_f}), implies that node 0 returns failure when either node 1 returns failure (no way to reach the safe area), or when node 1 returns success and node 9 returns failure (in safe area, but no way to get object to goal). 


\begin{definition}\emph{(Fallback Compositions of BTs)}
\label{def_fallback}
 Two or more BTs  can be composed into a more complex BT using a Fallback operator,
 $\bt_0=\mbox{Fallback}(\bt_1,\bt_2).$ 
 Then $r_0,u_0$ are defined as follows
\begin{align}
   \mbox{If }x\in {F}_1:& 
   &r_0(x) =  r_2(x),~ 
   u_0(x) =  u_2(x)  \label{bts:eq:fall1} \\ 
   \mbox{ else}:& 
   &r_0(x) =  r_1(x),~ 
   u_0(x) =  u_1(x) \label{bts:eq:fall2}
 \end{align}
\end{definition}

Note that when executing the new BT, $\bt_0$  first keeps executing its first child $\bt_1$, in (\ref{bts:eq:fall2}) as long as it returns  \emph{Running} or \emph{Success} $(x \not \in F_1)$.  
The second child of the Fallback is executed in (\ref{bts:eq:fall1}), only when the first returns \emph{Failure} $(x \in F_1)$. Finally, the Fallback itself $\bt_0$ returns \emph{Failure} only when all children have been tried, but failed $(x \in F_1 \cap F_2)$, hence the name Fallback.

 For notational convenience, we write
\begin{equation}
\mbox{Fallback}(\bt_1, \mbox{Fallback}(\bt_2,\bt_3))= \mbox{Fallback}(\bt_1,\bt_2, \bt_3),
\end{equation}
and similarly for arbitrarily long compositions. The fallback node is also denoted by ($?$), as seen in Figure~\ref{seq_fall}.

\begin{lemma}
 If $\bt_0=\mbox{Fallback}(\bt_1,\bt_2)$, then Definition~\ref{def_fallback}  implies that
 \begin{align}
 S_0 &= S_1 \bigcup (F_1 \cap S_2),\\
 F_0 &= F_1 \cap F_2, \label{eq_fall_f}\\
 R_0 &= R_1 \bigcup (F_1 \cap R_2),
 \end{align}
\end{lemma}
\begin{proof}
 A straightforward application of  the definition gives the result above.
\end{proof}

\begin{definition}[Condition]
 If a BT $\bt_i$ is such that $R_i= \emptyset$ we call it a Condition. Being a BT, it still has $u_i$ defined, but as we will see in Lemma~\ref{lem_omega_executes} below, that control will not be executed.
 \end{definition}
 
 Consider again the Mobile Manipulator example in Figure~\ref{seq_fall}. 
 \emph{Make sure in safe area (1)} is a fallback composition of nodes 2, 3. 
 Equation (\ref{bts:eq:fall1}) now states that node 1 executes \emph{Move to safe area if possible (3)}, $u_1(x)=u_3(x)$, only when \emph{In Safe Area (2)} returns failure, $x\in F_2$. 
 Furthermore, node 2 is a condition, $R_2=\emptyset$, thus if $x\not \in F_2$ we have $x \in S_2$ and success will be returned by node 1 up to node 0 which would then execute \emph{Make sure object at goal (9)} and so on.
 Furthermore, Equation (\ref{eq_fall_f}) indicates that the only way for node 1 to fail is if both node 2 and node 3 fails. That is \emph{Make sure in safe area (1)} only returns failure if both \emph{In safe area (2)} and \emph{Move to safe area if possible (3)} return failure.
 

Now we have all we need to create and execute BTs. In the next section we will explore the modularity of BTs, and then analyze under what circumstances the execution will converge to the success region.

 \section{Optimal Modularity} 
 \label{sec_modularity}

  One of the key advantages of BTs is their modularity, a property made possible by the fact that all subtrees on all levels of a BT have the same interface, given by Definition~\ref{def_bt}. However, as was shown in \cite{biggar2020modularity}, a deeper analysis  can be made, by extending a measure of modularity/complexity used in graph theory. In this section we give a very brief overview of the key theoretical results, showing that BTs have a so-called cyclomatic complexity of one, a fact that makes them \emph{optimally modular}, within a particular class of control structures.
 
 \begin{figure}[!h]
	\centering
	\includegraphics[width=11cm]{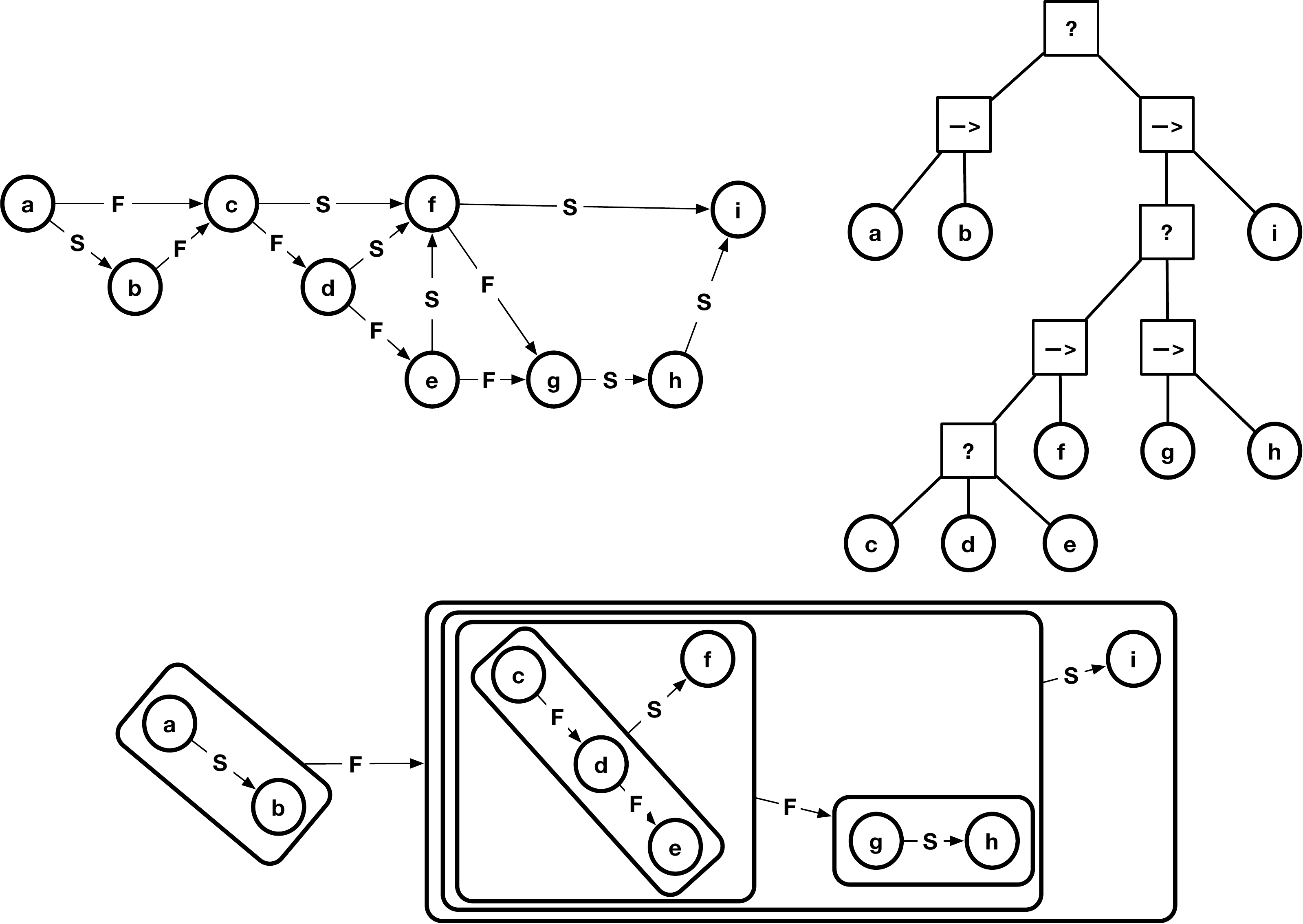}
	\caption{A BT (top right) with its corresponding DS (top left), and the module decomposition of the DS (bottom). Note how all graphs in the decomposition correspond to paths without cycles. By Theorem~\ref{thm_modular} all BTs will give rise to DSs with such non-cyclic graphs, a fact that results in all BTs having a  cyclomatic complexity of one.}
	\label{biggar_modularity}
\end{figure} 


To investigate the concept of modularity in general reactive control architectures, so-called  Decision Structures (DS) were defined in \cite{biggar2020modularity}. These are directed graphs, as illustrated in the upper left part of Figure~\ref{biggar_modularity}. Each node in this structure corresponds to a controller $u_i$, and each edge label corresponds to a return status $r_i$ that can be returned by the node the arc is leaving.  The DS is executed in the following way:  starting at the source node,  (a) in Figure~\ref{biggar_modularity} (top left), look at the return status $r_i(x)$ of that node, and follow the edge (if there is one) with a label corresponding to $r_i(x)$. Similarly, the return status of the new node is checked and the corresponding edge is followed until you find a $r_i(x)$ without a corresponding outgoing edge label. Then this controller $u_i$ is chosen. This process is then constantly iterated from the source to keep track of the proper controller to run.

Given the above it  can be seen in Figure~\ref{biggar_modularity} that the execution of the DS on the upper left is identical to the execution of the BT in the upper right.
Thus, this DS is equivalent to the BT.
 In fact,
DSs are generalizations of BTs in the sense that all BTs can be written as a DS, but not all DSs can be written as BTs. 
DSs are fairly general, as the label set can be of any size, as long as it is finite. Thus a DS is very similar to a FSM, with the difference that a DS constantly starts from the source in order to determine what controller to execute.

%

Using inspiration from modules in graph theory \cite{gallaiTransitivOrientierbareGraphen1967}, the authors of  \cite{biggar2020modularity} define modules in DS as follows.

\begin{definition}[Definition 6.3 in \cite{biggar2020modularity}, Modules in decision structures]
  Let $Z$ be a decision structure. Let $Y \subset N(Z)$ be a subset of the nodes where $Z[Y]$ is also a decision structure. We say $Y$ is a module if for every node $v \in N(Z) \setminus Y$, any arc from $v$ into $Y$ goes to $Y$’s source, and if there is an arc labelled $r$ out of $Y$ to $v$, then for every $y \in Y$ the $r$ out of $y$ exists and goes either to $v$ or to another element of $Y$.
\end{definition}
After that they define quotient DS, where the modules are collapsed into single nodes.
\begin{lemma}[Lemma 6.8 in \cite{biggar2020modularity}]
  Let $Z$ be a decision structure and $P$ a modular partition. Then the quotient $Z/P$ is also a decision structure. Moreover, if $P$ is maximal then $Z /P$ is prime.
\end{lemma}
Given these, a maximal module decomposition is defined (the interested reader is referred to \cite{biggar2020modularity}) leading to the following theorem.

\begin{theorem}[Theorem 6.15 in \cite{biggar2020modularity}]
Let $Z$ be a decision structure with $k$ distinct arc labels. Then $Z$ is structurally equivalent to a k-BT if and only if every quotient graph in Z's module decomposition is a path.\end{theorem}
\begin{proof}
See   \cite{biggar2020modularity}.
\end{proof}

Here, a k-BT is a generalization of BTs having a label set of size $k$, with 2-BTs corresponding to normal BTs, counting the labels success and failure, but not running since it does not lead to a transition in the DS. The k-BT also has k different interior nodes, a generalization of the two nodes Sequence and Fallback, defined for 2-BTs.

This theorem is clearly illustrated in Figure~\ref{biggar_modularity}, where we can see that there are no cycles in the module decomposition of the DS.
The number of cycles has been shown to be correlated with the difficulty of testing and debugging a piece of code
\cite{watsonStructuredTestingTesting1996}. Therefore, the concept of cyclomatic complexity has been defined for graphs and the authors of \cite{biggar2020modularity} extend it to DSs as follows:

\begin{definition}[Definition 6.19 in \cite{biggar2020modularity}]
  Let $Z$ be a decision structure. The cyclomatic complexity of $Z$ is the number of linearly independent undirected cycles in $Z$, plus one.
\end{definition}
Then, the concept is extended to account for modularity.
 \begin{definition}[Definition 6.21 in \cite{biggar2020modularity}]
  Let $Z$ be a decision structure. The essential complexity of $Z$ is the maximum cyclomatic complexity of any quotient graph in its module decomposition.
\end{definition}
 
 Finally they prove the following theorem.
 
\begin{theorem}[Theorem 6.23 in \cite{biggar2020modularity}]
\label{thm_modular}
Let $Z$ be a decision structure with $k$ distinct edge labels. $Z$ is equivalent to a k-BT if and only if it has essential complexity 1.
\end{theorem}
\begin{proof}
See   \cite{biggar2020modularity}.
\end{proof}

Looking at the case of $k=2$ (two distinct edge labels, S and F), the theorem says that BTs are exactly the DSs with essential complexity 1. Thus, BTs correspond to the class of optimally modular DSs.


 \section{Proving convergence} 
   \label{sec_convergence}
Many control problems are formulated in terms of making some equilibrium point stable, such that a wide set of state trajectories starting from different states all converge to that equilibrium point. For BTs we do not pick some particular equilibrium point, but instead assume that the success region of the root, $S_0$ is chosen to capture the set of desired outcomes, and therefore we also assume that  the design objective is to make a large set of state trajectories converge to points inside $S_0$.
Thus, in this section we will study the problem of when we can guarantee that the state will end up in $S_0$.

The main result is 
a general  convergence proof for BTs, Theorem~\ref{thm_main}, including a few examples.
We will try to make use of the modularity of BTs, in the sense that the result can be applied at all levels of abstractions, either treating an entire subtree as a single entity as in Figure~\ref{seq_fall}, or in terms of its parts, as in Figure~\ref{bt_big}.
 
The idea behind the proof is straightforward, and illustrated in Figure~\ref{operating_regions}.  As in Figure~\ref{operating_regions_0}, the state space is partitioned into operating regions, $\Omega_i$, and overall failure and success regions, $F_0,S_0$, while the arrows indicate possible transitions between these sets. If all transitions from some $\Omega_i$ are to either $S_0$ or $\Omega_j, j>i$, and the state never stays indefinitely in $\Omega_i$, it will eventually reach $S_0$. Note that a similar analysis can be done at several different levels of abstraction. If $\Omega_6 =\Omega_4 \cup \Omega_5$ the analysis can either be done considering $\Omega_4, \Omega_5$ separately as in Figure~\ref{operating_regions}(a), or together, as in Figure~\ref{operating_regions}(b).

 \begin{figure}[!h]
	\centering
	\begin{subfigure}[b]{0.35\textwidth}
         \centering
         \includegraphics[width=\textwidth]{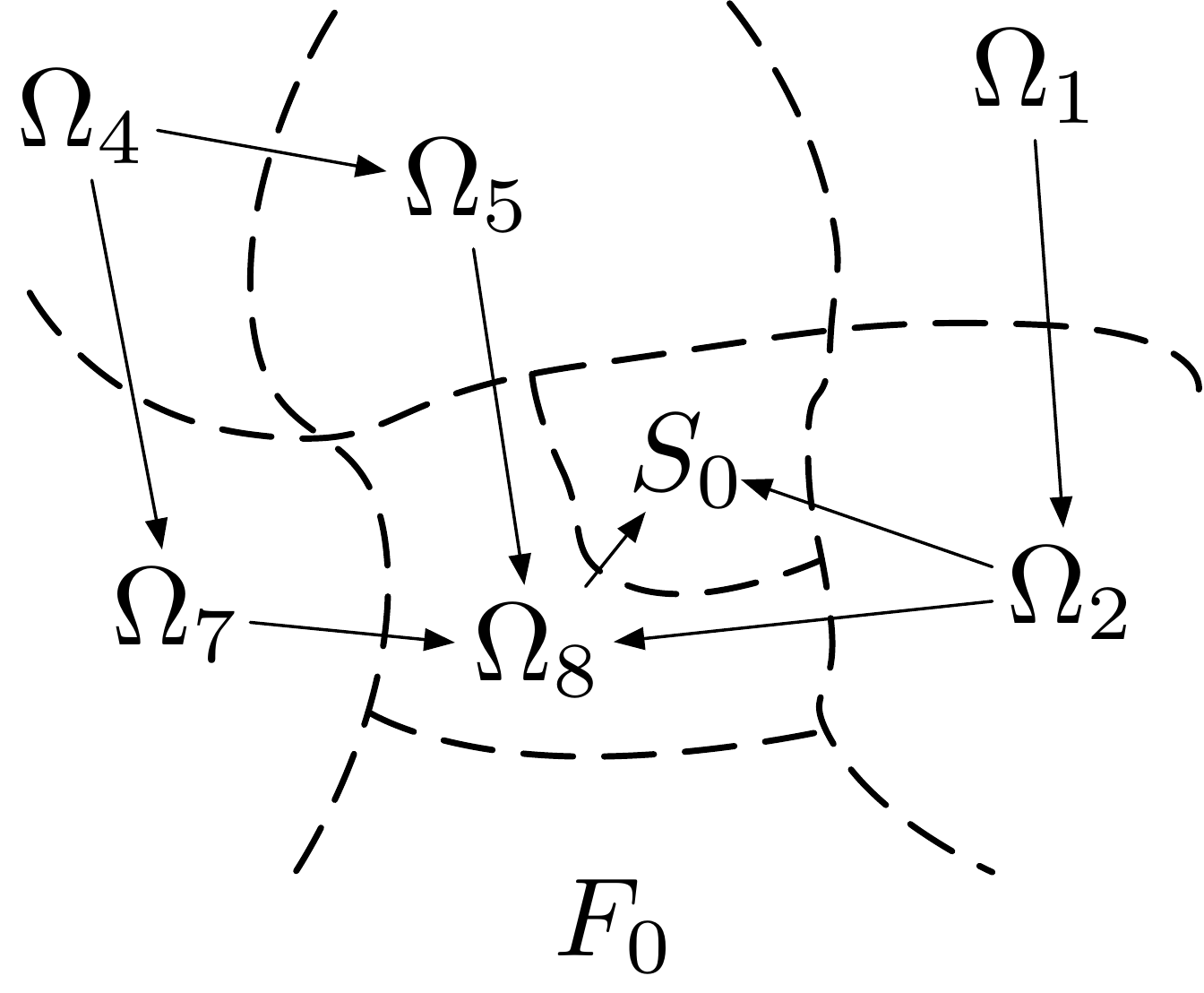}
         \caption{}
      \end{subfigure}
     \hfill
     \begin{subfigure}[b]{0.35\textwidth}
         \centering
         \includegraphics[width=\textwidth]{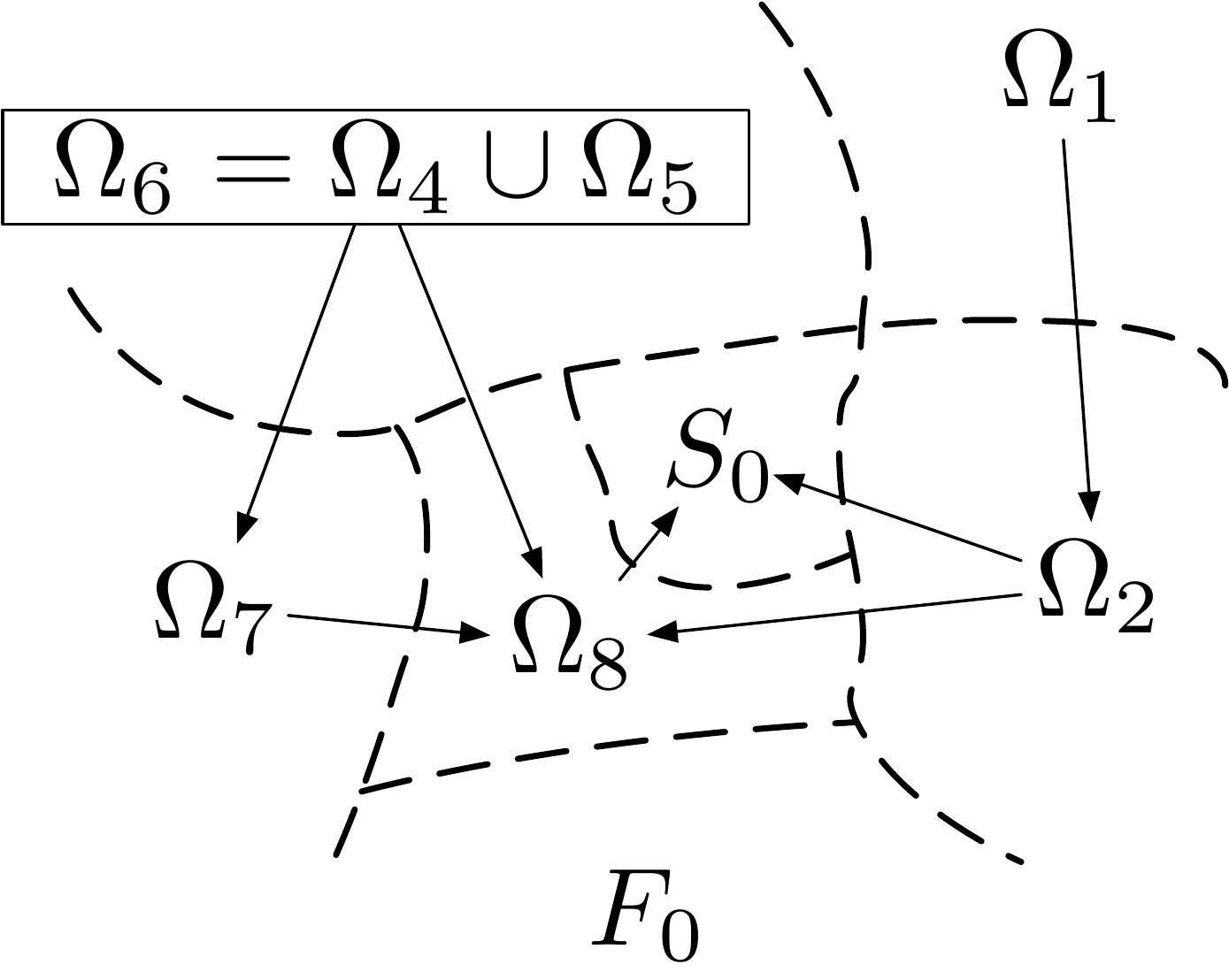}
                  \caption{}
     \end{subfigure}
     \vspace{5mm}
	\caption{The idea behind Theorem~\ref{thm_main}. If the state only transits between sets connected by arrows, and never stays indefinitely in any $\Omega_i$, it will eventually reach $S_0$.}	
	\label{operating_regions}
\end{figure} 

\subsection{The general result}

The sets $\Omega_i$ are defined below, and correspond to the regions where $u_i$ is running, see Lemma~\ref{lem_omega_executes}. But, in order to define $\Omega_i$, we first need to define the influence regions $I_i$, the parts of the state space where a change of $\bt_i$ might alter $u_0$, and to define  $I_i$ we need notation for parents and older siblings of a node. Some of these results are taken from \cite{colledanchiseHowBehaviorTrees2017,ogrenConvergenceAnalysisHybrid2020,sprague2021ctrl}.

\begin{definition}[parent and big brother of a node]
Given a node $i$, let $p(i)$ be the parent of  the node and
 $b(i)$ be the closest sibling to the left (the big brother) of the node.
 \end{definition}
 
 \noindent
 Note that $p(i)$ is undefined if $i$ is the root, and $b(i)$ is undefined if there is no sibling to the left.

\begin{definition}[Influence region]
The Influence Region $I_i$ of node $i$ 
 is defined as follows
\begin{align}
I_i &= X &\mbox{If } i \mbox{ is the root } \\
I_i &= I_{p(i)}  &\mbox{If } i  \mbox{ is the leftmost sibling and }\exists p(i)\\
I_i &= I_{b(i)} \cap S_{b(i)} &\mbox{If } p(i)  \mbox{ is a Sequence and } \exists b(i)  \label{eq_influence_seq}\\
I_i &= I_{b(i)} \cap F_{b(i)} &\mbox{If } p(i)  \mbox{ is a Fallback and  } \exists b(i) \label{eq_influence_fall}
 \end{align}
 \end{definition}

 Note that the influence region $I_i$ is the part of the state space where the design of $\bt_i=(u_i,r_i)$ influences the execution of $\bt_0$.
 Also note that $I_i$ is fundamentally different from $S_i, F_i, R_i$ in the sense that $I_i$ depends entirely on the part of $\bt_0$ that is outside $\bt_i$, the parent and siblings, while, on the other hand $S_i, F_i, R_i$ depends entirely on what is inside $\bt_i$ itself, $u_i(x)$.

\begin{lemma}
\label{lem_influence_region}
If $x \not \in I_i$ then changing the implementation of $u_i,r_i$ will not change the value of $u_0(x)$.
\end{lemma}
\begin{proof}
We will use Lemma~\ref{lem_omega_executes} below, that shows that if $x \in \Omega_i = I_i \cap R_i$, then $u_0(x)=u_i(x)$.
To maximize the influence of $u_i,r_i$ we change $r_i(x)$ to always return running, making $R_i=X$ and $\Omega_i = I_i$.
However, if $x \not \in I_i$ we have $x \not \in \Omega_i$ and another subtree is still controlling the execution.
\end{proof}

As seen above, $I_i$ depends on external factors and $S_i, F_i, R_i$ depends on internal factors. We will now define $\Omega_i$ such that $\Omega_i$ is the region where $u_0(x) \equiv u_i(x)$ and $r_i(x)=R$, that is the region where $\bt_i$ is controlling the execution.
\begin{definition}[Operating region]
\label{def_operating_region}
The Operating Region $\Omega_i$ of node $i$ 
 is defined as follows
 
\begin{equation}
 \Omega_i = I_i \cap R_i
\end{equation}

 \end{definition}

\begin{lemma}
\label{lem_operating_region}
 For a given node $j$, the operating regions of the children is a partitioning of $\Omega_j$, that is 
 \begin{align}
  \bigcup_{i: p(i)=j} \Omega_i &= \Omega_j, \\
   \Omega_i \cap \Omega_k &= \emptyset, \forall i,k: i\neq k, p(i)=p(k)
\end{align}
\end{lemma}
\begin{proof}
Let the parent index be 0 and the two children indices be 1 and 2.
We need to show that this holds for both Sequence and Fallback compositions.
If the parent node is a Sequence we have that $R_0 = R_1 \bigcup (S_1 \cap R_2)$.
For the influence regions, assume that $I_0$ is given, which gives $I_1= I_0$ and $I_2= I_1 \cap S_1=I_0 \cap S_1$.
Thus we have 
 \begin{align}
 \Omega_0 &= I_0 \cap R_0 = I_0 \cap (R_1 \bigcup (S_1 \cap R_2)) \\
 \Omega_1 &= I_1 \cap R_1 = I_0 \cap R_1  \\
  \Omega_2 &= I_2 \cap R_2 = I_0 \cap S_1 \cap R_2 
\end{align}
This gives $\Omega_1 \cap \Omega_2 = I_0 \cap R_1 \cap  I_0 \cap S_1 \cap R_2 = \emptyset$, since $R_1 \cap S_1=\emptyset$.
Furthermore, 
$\Omega_1 \cup \Omega_2 = (I_0 \cap R_1) \cup  (I_0 \cap S_1 \cap R_2) =
I_0 \cap (R_1 \cup (S_1 \cap R_2))=
I_0 \cap R_0=
  \Omega_0$.
  
Similarly, if the parent node is a Fallback we have that $R_0 = R_1 \bigcup (F_1 \cap R_2)$.
For the influence regions, assume that $I_0$ is given, which gives $I_1= I_0$ and $I_2= I_1 \cap F_1=I_0 \cap F_1$.
Thus we have 
 \begin{align}
 \Omega_0 &= I_0 \cap R_0 = I_0 \cap (R_1 \bigcup (F_1 \cap R_2)) \\
 \Omega_1 &= I_1 \cap R_1 = I_0 \cap R_1  \\
  \Omega_2 &= I_2 \cap R_2 = I_0 \cap F_1 \cap R_2 
\end{align}
This gives $\Omega_1 \cap \Omega_2 = I_0 \cap R_1 \cap  I_0 \cap F_1 \cap R_2 = \emptyset$, since $R_1 \cap F_1=\emptyset$.
Furthermore, 
$\Omega_1 \cup \Omega_2 = (I_0 \cap R_1) \cup  (I_0 \cap F_1 \cap R_2) =
I_0 \cap (R_1 \cup (F_1 \cap R_2))=
I_0 \cap R_0=
  \Omega_0$.
\end{proof}

\begin{lemma}
\label{lem_operating_region_leaves}
 For a given subtree, the operating regions of the leaves is a partitioning of the operating region of the root.
 \end{lemma}
\begin{proof}
 A recursive application of Lemma \ref{lem_operating_region}.
\end{proof}

As described above, we want to enable the convergence analysis to be done at different levels of abstraction, as illustrated in Figure~\ref{operating_regions}. Thus we make the following definition.

\begin{definition}
 A level of abstraction $L \subset \{0, 1, 2, \ldots\}$ is a set of indices such that 
 \begin{equation}
 X = \bigcup_{i \in L} \Omega_i \cup S_0 \cup F_0, 
\end{equation}
and $\Omega_i \cap \Omega_j = \emptyset, \forall i \neq j \in L$.
\end{definition}

\begin{lemma}
 The root is one level of abstraction, with $L=\{0\}$ and all the leaves is another level of abstraction with $L=\{\mbox{indices of leaves}\}$.
 \end{lemma}
\begin{proof}
For the root we have $\Omega_0=I_0 \cap R_0= X \cap R_0=R_0$ this makes $\bigcup_{i \in L} \Omega_i \cup S_0 \cup F_0 = R_0 \cup S_0 \cup F_0=X$. 
If it holds for the root it must also hold for the leaves, since $\bigcup_{i \in L} \Omega_i = \Omega_0$
by Lemma \ref{lem_operating_region_leaves}.
\end{proof}

\begin{lemma}
\label{lem_omega_executes}
 If $\bt_i$ is a subtree of $\bt_0$, then $x \in \Omega_i$, implies that $u_0(x)=u_i(x)$, 
 that is controller $i$ is executed while inside $\Omega_i$ 
 \end{lemma}
\begin{proof}
It holds for the root since $\Omega_0=I_0 \cap R_0=R_0$.
We will now show that if it holds for a parent, it will also hold for a child.
Assume it holds for the parent $j=p(i)$. It remains to show that $u_j(x)=u_i(x)$. We know that $x \in \Omega_i \subset \Omega_j$ and $\Omega_i = I_i \cap R_i$. 
If $j$ is the leftmost child then $x \in R_i$ implies $x \not \in S_i \cup F_i$ which gives $u_j(x)=u_i(x)$ by equation  (\ref{bts:eq:seq1}) and  (\ref{bts:eq:fall1}).
Assume the parent is a Sequence node.
If $j$ is not the leftmost child then $x \in I_i$ implies $x \in S_{b(i)}$ which gives $u_j(x)=u_i(x)$ by Equation  (\ref{bts:eq:seq2}) and (\ref{eq_influence_seq}).
Conversely, assume the parent is a Fallback node.
If $j$ is not the leftmost child then $x \in I_i$ implies $x \in F_{b(i)}$ which gives $u_j(x)=u_i(x)$ by Equation (\ref{bts:eq:fall2}) and (\ref{eq_influence_fall}).

\end{proof}

Given these concepts, we can now formulate our main theorem on convergence of BTs, inspired by \cite{burridgeSequentialCompositionDynamically1999,connerIntegratedPlanningControl2006,reistSimulationbasedLQRtreesInput2010}. The idea is that if
the state moves through the operating regions $\Omega_i$ in strictly increasing order, without staying longer than $\tau$ in any region, and the only other allowed region is $S_0$, then the system will reach $S_0$ in finite time. Formally, we write

\begin{theorem}[Convergence of BTs]
\label{thm_main}
Given a BT, an external constraint region $\bar C\subset X$  that is to be kept invariant, and a level of abstraction $L$.
 If there exists a re-labelling of the $N$ nodes in $L$ such that

\begin{equation}
\label{eq_ci_invariant}
C_i =  \bigl( \bigl(\bigcup_{j \in L, j \geq i} \Omega_j\bigr) \cup S_0\bigr) \cap \bar C
\end{equation}
is invariant under $u_i$ for all $i \in L$, and there exists a $\tau > 0$ such that if $x(t)\in \Omega_i$ then $x(t+\tau) \not \in \Omega_i$, then there exist a time $t' \leq N\tau$ such that if $x(0)\in C_1$, then $x(t') \in S_0$.
\end{theorem}

\begin{proof}

We have that if $x(t)\in \Omega_i$ then $x(t+\tau) \not \in \Omega_i$, but $C_i$ is invariant under $u_i$, so either $x(t+\tau)\in \Omega_j, j>i$ or 
$x(t+\tau)\in S_0$. Thus there can be at most $N$ transitions before $x(t)\in S_0$, in total taking at most time $N\tau$.
\end{proof}

By saying that a set $B \subset X$ is invariant under $u_i$ we mean that if $x(0) \in B$ and $\dot x = f(x,u_i(x))$, then $x(t) \in B$ for all $t>0$, and similarly for discrete time executions with $x_{t+1} = f(x_t,u_i(x))$.

\begin{remark2}
 Note that the challenge in proving convergence for BTs now lies in choosing a level of abstraction, re-ordering the nodes in $L$, and designing $u_i$ such that $C_i$ is invariant and $x(t+\tau) \not \in \Omega_i$.
\end{remark2}

\begin{remark2}
The purpose of the external constraint region $\bar C$ is to enable separate analysis of a BT that is then being used as a subtree of another BT. It this is not needed, set $\bar C = X$.
 \end{remark2}

The deterministic analysis above can be complemented with a probabilistic result from \cite{paxtonRepresentingRobotTask2019}.

\begin{lemma}[Probabilistic transitions]
 If the execution of Definition~\ref{def_execution} was replaced by non-deterministric transitions, and the controllers $u_i$ are such that undesired transitions, from an $\Omega_i$ to an $\Omega_j$ with $j<i$, happen with a probability $1-p^i$, with $0< p \leq p^i < 1$, then the expected number of transitions $T$ before reaching $S_0$ is bounded above, $E(T) \leq N/p^N$, and the probability of reaching the goal with at most $k$ transitions is $P_k=1-\gamma^{k+1}$, with $\gamma=1-p^N$, which makes $P_\infty= 1$.
\end{lemma}
\begin{proof}
 See  \cite{paxtonRepresentingRobotTask2019}.
\end{proof}

The Lemma above can be interpreted in two ways. Either for nondeterministic executions, as stated in the Lemma, or for deterministic executions, where an un-modeled external agent moves things around, thereby causing a finite set of jumps in the state. Also note that the expected number of transitions gives an upper bound on the expected convergence time of $\tau N/p^N$.

\subsection{Three examples}
We will now apply Theorem \ref{thm_main}, to the sequence of desired goals in Figure~\ref{proof_illustrations}(a), a fallback of actions where each one is designed to satisfy the preconditions of the one to the left in Figure~\ref{proof_illustrations}(b)
 and the more complex mobile manipulation BT from Figure~\ref{bt_big}.
The resulting sets are shown in Table \ref{tab_ex_thm}.
As can be seen, the sets $C_i$ to be kept invariant are often not that complex, and creating controllers $u_i$ to satisfy them is is often reasonable.

 \begin{figure}[!h]
	\centering
	\begin{subfigure}[b]{0.45\textwidth}
         \centering
         \includegraphics[width=\textwidth]{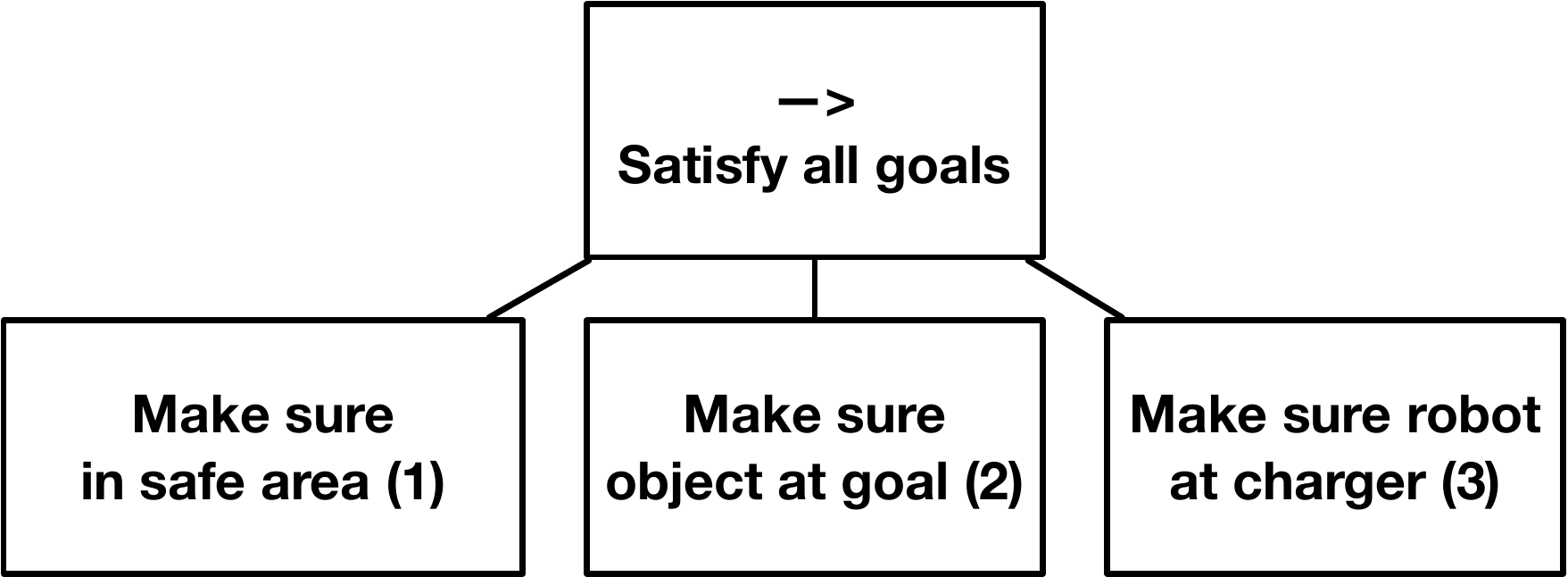}
         \caption{}
      \end{subfigure}
     \hfill
     \begin{subfigure}[b]{0.45\textwidth}
         \centering
         \includegraphics[width=\textwidth]{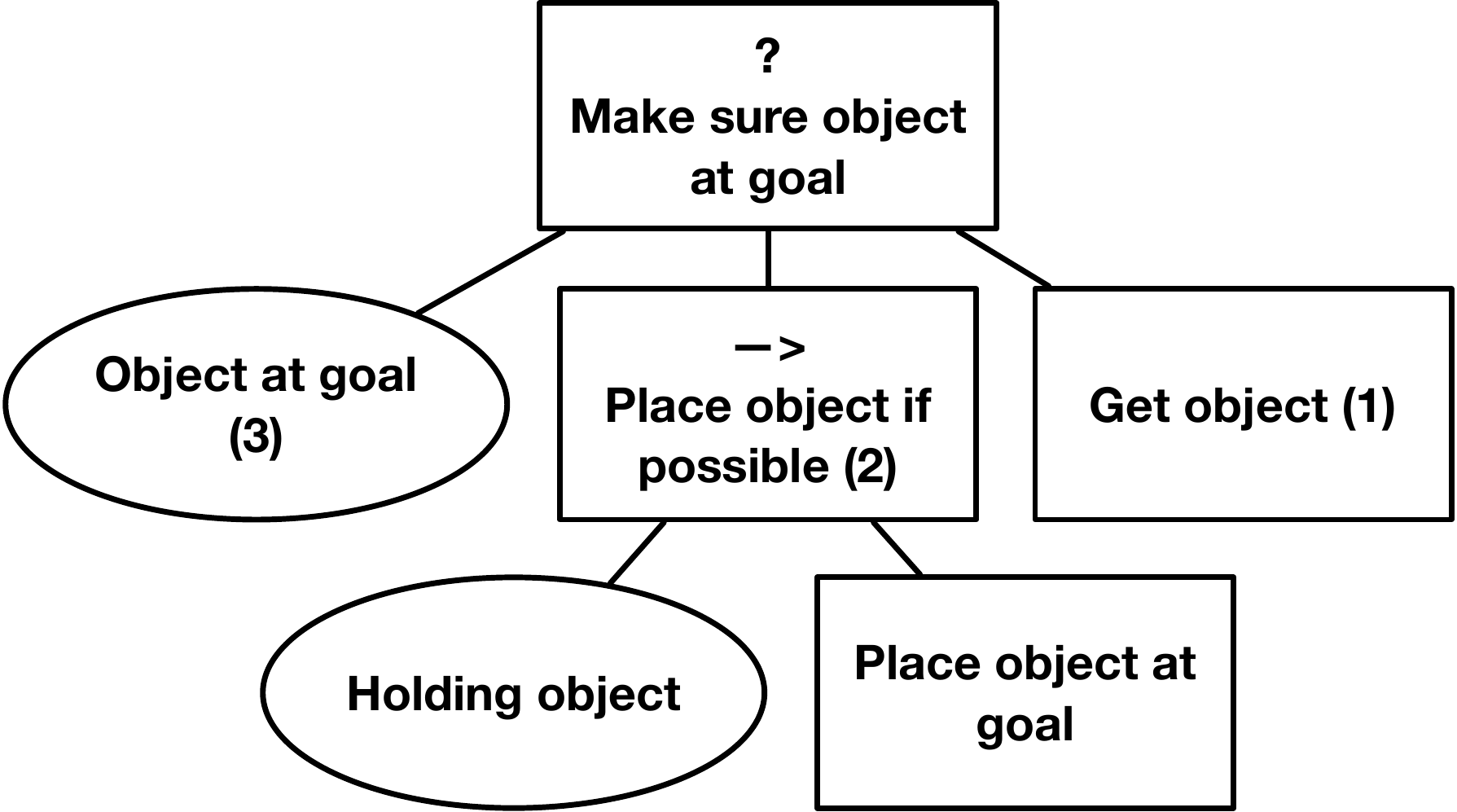}
                  \caption{}
     \end{subfigure}
     \vspace{5mm}
	\caption{Two examples to apply Theorem \ref{thm_main} on. Note that the index labels on the nodes are going from left to right for the sequence, and right to left for the fallback. This is a result of the re-labeling that is done when applying Theorem~\ref{thm_main}.}	
	\label{proof_illustrations}
\end{figure}

\begin{lemma}[Standard sequence]
If $\bt_0=\mbox{Sequence}(\bt_1, \ldots, \bt_N)$, we have

\begin{align}
I_i &= \bigcap_{j < i} S_{j} \cap I_{p(i)} \\
 \Omega_i &=  \bigcap_{j < i} S_{j} \cap I_{p(i)} \cap R_i  \\
 C_i &= (\bigcap_{j < i} S_{j} \cup S_0) \cap \bar C
 \end{align}
\end{lemma}

\begin{proof}
$I_i, \Omega_i$ are clear from the definition of a sequence node, only executing the next child if the previous succeeds.
From Theorem~\ref{thm_main} we have that $C_i = \bigl( \bigl(\bigcup_{j \geq i} \Omega_j\bigr) \cup S_0\bigr) \cap \bar C$.
Thus we need to show that $\bigcup_{j \geq i} \Omega_j = \bigcap_{j < i} S_{j}$.
Assume $x\in \bigcap_{j < i} S_{j} $ then $x \in S_j \forall j<i$ thus $x \not \in \Omega_j \forall j<i$ thus $\exists k\geq i: x \in \Omega_k$, since $\Omega_i$ is a partition of $\Omega_{p(i)}$.
Conversely, if $x\in \bigcup_{j \geq i} \Omega_j$ then $\exists k\geq i: x \in \Omega_k$ and thus $x\in \bigcap_{j < i} S_{j} $.
\end{proof}

\begin{table}
\caption{Three examples of applying Theorem \ref{thm_main}.}
\label{tab_ex_thm}
 \begin{center}
\begin{tabular}{ |p{2cm}|  p{3cm}|  p{3cm}|  p{3cm}| }
\hline
 Name & Objective & Operating region $\Omega_i$ & To keep invariant $C_i$ \\ 
 \hline
 \hline
  Figure \ref{proof_illustrations}(a)&&&\\
 \hline
 $\bt_1$:  Make sure in safe area &  \mbox{in safe area} & $\lnot \mbox{in safe area}$ & $\emptyset$ \\ 
 \hline
 $\bt_2$: Make sure object at goal &  \mbox{object at goal} &$\mbox{in safe area} \land \lnot \mbox{object at goal}$ & $\mbox{in safe area}$ \\  
 \hline
 $\bt_3$: Make sure robot at charger&  \mbox{robot at charger} &$\mbox{in safe area} \land  \mbox{object at goal}\land \lnot  \mbox{robot at charger}$ & $\mbox{in safe area} \land  \mbox{object at goal}    $\\
\hline
\hline
 Figure \ref{proof_illustrations}(b)&&&\\
 \hline
 $\bt_1$ : Get object 			&  \mbox{Holding object}& $\lnot \mbox{Holding object} \land\lnot \mbox{Object at goal}$ & $\emptyset$ \\ 
 \hline
 $\bt_2$ : Place object if  possible &  \mbox{Object at goal}	& $\mbox{Holding object} \land\lnot \mbox{Object at goal}$ & $\mbox{Holding object}$ \\  
\hline
\hline
 Figure \ref{bt_big}&&&\\
 \hline
 $\bt_5$:  Move to safe area &  \mbox{in safe area} & $\lnot \mbox{in safe area}$ & $\emptyset$ \\ 
 \hline
 $\bt_8$:  Recharge &  \mbox{Proper battery level} & $ \mbox{in safe area} \land \lnot \mbox{proper battery level}$ & $\mbox{in safe area}$ \\ 
 \hline
 $\bt_{19}$:  Move to Object &  \mbox{Robot near object} & $ \mbox{in safe area} \land \mbox{proper battery level} \land \lnot \mbox{object at goal} \land \lnot \mbox{object in gripper} \land \lnot \mbox{robot near object} \land \mbox{Free path to object exists}$ & $\mbox{in safe area} \land  \mbox{proper battery level}$ \\ 
 \hline
  $\bt_{20}$:  Grasp object with left arm &  \mbox{Object in gripper} & $ \mbox{in safe area} \land \mbox{proper battery level} \land \lnot \mbox{object at goal} \land \lnot \mbox{object in gripper} \land \mbox{robot near object}$ & $\mbox{in safe area} \land  \mbox{proper battery level}$ \\ 
  \hline
    $\bt_{21}$: (skipped for lack of space) &&& \\ 
  \hline
    $\bt_{32}$:  Move to goal &  \mbox{Robot near goal} & $ \mbox{in safe area} \land \mbox{proper battery level} \land \lnot \mbox{object at goal} \land \mbox{object in gripper} $ & $\mbox{in safe area} \land  \mbox{proper battery level} \land  \mbox{object in gripper}$ \\ 
  \hline
    $\bt_{33}$:  Place object at goal&  \mbox{object at goal} & $ \mbox{in safe area} \land \mbox{proper battery level}  \land \lnot \mbox{object at goal} $ & $\mbox{in safe area} \land  \mbox{proper battery level} $ \\ 
  \hline
   $\bt_{36}$:  Ask other agent to place object&  \mbox{object at goal} & $ \mbox{in safe area} \land \mbox{proper battery level}  \land \lnot \mbox{object at goal}  \land  \mbox{subtree $\bt_{11}$ returned failure} $ & $\mbox{in safe area} \land  \mbox{proper battery level}  \land  \mbox{subtree $\bt_{11}$ returned failure}$ \\ 
  \hline
    $\bt_{39}$:  Move to charger&  \mbox{robot at charger} & $ \mbox{in safe area} \land \mbox{proper battery level}  \land  \mbox{object at goal}  $ & $\mbox{in safe area} \land  \mbox{proper battery level}\land  \mbox{object at goal} $ \\ 
  \hline
\end{tabular}
\end{center}
\end{table}

For the next example we use the design principle \emph{implicit sequence}
\cite{colledanchiseBehaviorTreesRobotics2018,paxtonRepresentingRobotTask2019},
where each child of a fallback node is designed to satisfy the precondition of a sibling to its left, while not failing, see $\hat C_i\subset C_i$ in Equation (\ref{eq_cibar}) below. Therefore, the numbering is also done from right to left as part of applying Theorem~\ref{thm_main}.

\begin{lemma}[Implicit sequence]
If $\bt_0=\mbox{Fallback}(\bt_N, \ldots, \bt_1)$, and forall $i<N$ there is $j>i$ such that $S_j \cup R_j \supset S_i$
we have that
\begin{align}
I_i &= \bigcap_{j > i} F_{j} \cap I_{p(i)} \\
 \Omega_i &=  \bigcap_{j > i} F_{j} \cap I_{p(i)} \cap R_i  \\
 C_i \supset \hat C_i&= (R_i \cup S_i ) \cap \bar C   \label{eq_cibar}
\end{align}
\end{lemma}

\begin{proof}
$I_i, \Omega_i$ are clear from the definition of a Fallback node, only executing the next child if the previous fails.
Furthermore, $S_0=S_N$, since $S_i \cap I_i = \emptyset$ for $i \neq N$, due to $S_j \cup R_j \supset S_i$ for some $j>i$.

From Theorem~\ref{thm_main} we have that $C_i = \bigl( \bigl(\bigcup_{j \geq i} \Omega_j\bigr) \cup S_0\bigr) \cap \bar C$.
Thus we need to show that $\bigcup_{j \geq i} \Omega_j \cup S_N \supset R_i \cup S_i \cup S_N$.
Assume $x \in R_i \cup S_i \cup S_N$ then $x \not \in F_i$ and hence $x \not \in \bigcup_{j < i} \Omega_j $. Therefore 
 $x \in \bigcup_{j \geq i} \Omega_j $.
\end{proof}

Using the results above we can compute the sets $C_i$ that needs to be kept invariant by the controller $u_i$ for each region $\Omega_i$.
Then a table like the one in Table \ref{tab_ex_thm} can be created, and used as design specification for the~$u_i$.

Table \ref{tab_ex_thm} also includes the results for the more complex BT in Figure~\ref{bt_big}, with node numbers given by a depth first traversal of the tree.
We will now see how such a BT can be created recursively from a list of actions, with corresponding precondition and postconditions.

%
%
%

 \section{A design principle exploiting modularity and feedback} 
  \label{sec_design}

\begin{figure}[!h]
	\centering
	\includegraphics[width=12cm]{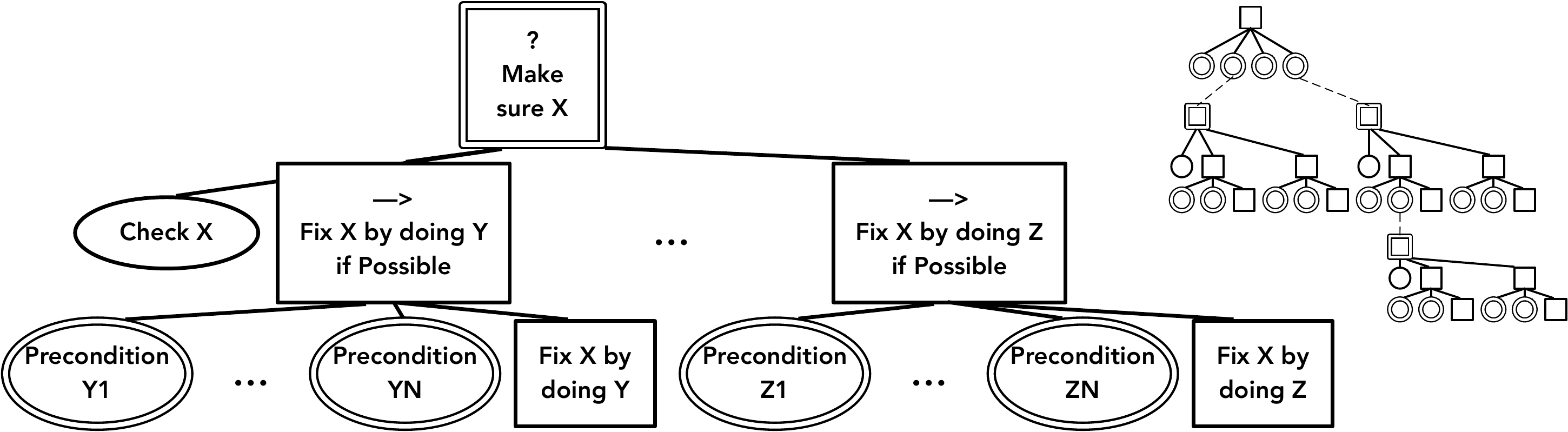}
	\caption{On the left we have a basic BT for achieving some condition X by either executing action Y or action Z. Note how the BT in Figure~\ref{bt_big} is created by connecting eight BTs on this form. This was done by starting with a sequence with the four top priority goals, as illustrated on the right, and then recursively replacing (dashed lines) a leaf condition with the root of a BT achieving that very condition. Both conditions that can be replaced, and Fallback nodes that have replaced a condition are double stroked in this figure as well as Figure~\ref{bt_big}.}
	\label{bt_ppa}
\end{figure} 

In this section we will give a concrete example on the use of hierarchical modularity and feedback in BTs.
Consider the example BT in Figure~\ref{bt_ppa}, designed to make sure some condition X holds.
If X is true it will immediately return success. If not, it will act to make X true, hence the name \emph{Make sure X}.
Sometimes there are multiple ways of making X true, in those cases the options are collected under a fallback node, so if one option (say Y) fails, or is not applicable, another option (say Z) can be invoked. Both of Y and Z have their own preconditions, describing when they can be invoked, as illustrated in the figure.

The key idea is now to recursively apply the design on the left of Figure~\ref{bt_ppa}, as illustrated on the right of the figure.
First we list four top priority goals in a sequence node. Then, 
instead of just checking the conditions we can replace them with a small BT of the form to the left that tries to make them true.
The resulting BT will have some new conditions, which in turn can be replaced by small BTs and so on.
All conditions that can be replaced, or have been replaced, are marked with double strokes in both Figure~\ref{bt_big} and Figure~\ref{bt_ppa}.
Note that it does not make sense to replace the single stroked conditions, such as \emph{Check X}, as these already have actions for achieving them.

This recursive approach is a good example of the hierarchical modularity made possible by BTs.
The design of  Figure~\ref{bt_ppa} is only about achieving X, not about why, or what will happen later.
It also illustrates the use of feedback.
The BT first checks if something needs to be done, if not everything is fine. If something needs to be done it tries to achieve it, and if one option failed, another one is applied.

A detailed  analysis of the design above can be found in \cite{ogrenConvergenceAnalysisHybrid2020}, including a discussion on when it works and when it does not work. Here we note that applying Theorem~\ref{thm_main} we get the sets of Table \ref{tab_ex_thm}.
Remember that the execution is supposed to progress in increasing order of index labels, thus the table can be read from top to bottom. Note that the sets $C_i$ mostly correspond to not violating previously achieved subgoals, such as \emph{in safe set} and \emph{proper battery level}, but when carrying the object to the goal it also includes \emph{object in gripper}. Perhaps the most surprising item is $C_{36}$ for $\bt_{36}$, including \emph{subtree $\bt_{11}$ returns failure}. This is needed to avoid cases where one option first fails, but during the execution of a fallback option, the first option is somehow activated again before it fails yet again, causing a switching back and forth, see \cite{ogrenConvergenceAnalysisHybrid2020} for details.

\section{Guaranteeing safety and invariance using control barrier functions} 
\label{sec_cbf}
In this section we will see how Control Barrier Functions (CBFs) can sometimes be used to provide the invariance guarantees we need in Theorem~\ref{thm_main}.
Furthermore, we will see how the standard way of using CBFs to guarantee safety appears as a special case of Theorem~\ref{thm_main}.
Finally, we will see how to handle conflicting objectives, as trying to keep several sets invariant might not always be possible.

\subsection{Control Barrier Functions}

As seen in Section \ref{sec_convergence} above, the key aspect of Theorem~\ref{thm_main}
 is for $u_i$ to keep the set $C_i$ invariant. Fortunately, keeping sets invariant is the main objective of CBFs.
 Below we will assume that the execution in Definition \ref{def_execution} runs in continuous time, but corresponding discrete time concepts can also be created.

The key idea behind CBFs \cite{amesControlBarrierFunctions2019,ogrenAutonomousUCAVStrike2006} is to
specify a barrier function $h:X \rightarrow \mathbb{R}$ such that the so-called \emph{safe set} $\mathcal{C}$ is characterized by: 
$
\mathcal{C} = \{x\in X: h(x) \geq 0\}.
$
Given the continuous system dynamics $\dot x = f_C(x,u)$, if we choose controls $u \in U_{inv}$ where
\begin{align}
\label{eq:CBF}
U_{inv}=\bigl \{u \in U: \frac{dh}{dx}f_C(x,u) \geq - \alpha(h(x)) \bigr \},
\end{align}
and $\alpha \in \mathcal{K}$ is a class K functions \cite{amesControlBarrierFunctions2019}, we are guaranteed to stay in the safe set $x \in \mathcal{C}$.

\subsection{Guaranteeing safety and handling conflicting objectives}
Guaranteeing safety is often of highest priority. In this section we will see how we can use CBF to address the invariance property of Theorem~\ref{thm_main}, in a way that includes safety guarantees as a special case, as described in \cite{ozkahramanCombiningControlBarrier2020}.

Since we need the conditions $C_i$ to be invariant, we first make the following assumption:

\begin{assumption}
 Each condition $C_i:X \rightarrow \{0,1\}$ can be formulated in terms of a CBF $h_i$, see Equation (\ref{eq:CBF}), as follows 
  \begin{align}
  {C}_i &= \{x\in X:h_i(x) \geq 0\}
  \end{align}
\end{assumption}

Having just one CBF we can guarantee invariance, but if there are several, they might represent conflicting objectives, such as  \emph{in safe area} and \emph{at charger}, if the charger happens to be located outside the safe set.
Normally, the intersection of the corresponding control sets $U_{inv}$ would guarantee invariance of all sets, but 
if the objectives are conflicting, the intersection might be empty.
 In these cases we will make use of the fact that the BT includes a clear priority order of the objectives, e.g., that \emph{in safe area} is more important than \emph{at charger}. 
 The idea is then to include as many sets as possible in the intersection, while
 still making sure the intersection is non-empty.

Thus, we define the following sets of controls, 
where $U_i \subset U$ guarantees invariance of $C_i$, 
$\bar U_i \subset U$ guarantees invariance of all $C_j, j\leq i$ (but might be empty) and
$\hat U_i \subset U$ guarantees invariance of some of the $C_j, j\leq i$ (but is guaranteed to be non-empty).
\begin{definition}
 Let 
 \begin{align}
U_i &= \{u \in U: \frac{dh_i}{dx}f(x,u) \geq - \alpha(h_i(x))\} \label{eq:CBFi} \\
\bar U_i  &= \bigcap_{j=1}^i U_j \\
\hat U_i &= \bar U_j : j \leq i, \bar U_j \neq \emptyset \land (j = i \lor \bar U_{j+1} = \emptyset)  
\end{align}
\label{def:K}
\end{definition}
\vspace{-5mm}
We can now choose a control inside $\hat U_{i}$
that is as close as possible to some other desired value $w_i(x)$ that is designed to reach the current subgoal, as in the CBF-QP of \cite{amesControlBarrierFunctions2019},
\begin{align}
u_i = \mbox{ argmin}_u& ||u-w_i(x)||^2  \label{eq:local_ctrl} \\
\mbox{s.t.  } & u \in \hat U_{i} \nonumber 
\end{align}

If we apply this approach to an arbitrarily complex BT, such as the one in Figure \ref{bt_big}, with a safety objective as first priority,  the CBF approach above will guarantee that we will never violate this objective. In the best of worlds, we might achieve all objectives, but we know that the robot will always be safe.

\section{Explainable AI and human robot interaction} 
\label{sec_exai}

As robots  share workspaces with humans to an increasing extent, questions regarding human robot interaction become more important. Safety, as seen above, is often most important, but to achieve efficient interaction it is also important that the human can predict, trust and understand the robot.

In \cite{endsleyAutonomousHorizonsSystem2015} a number of guidelines for trustworthy autonomy are mentioned. These include that the system should be transparent and traceable, in the sense that
\enquote{the system must be able to explicitly explain its reasoning in a concise and usable format (either visual or textual)}.
As illustrated in Figure~\ref{bt_big}, this requirement is satisfied by BTs in the sense that at any point, you can find the leaf node that is executing and follow the branch  all the way up to the root to see why this subtree is executing. If the recursive backward chained approach described in Section~\ref{sec_design} is used, reading 
the expanded preconditions (double stroked in Figure~\ref{bt_big}) we see that the robot is currently executing Move to object, (in order to) Make sure robot near object, (to) make sure object in gripper, (to) make sure object at goal.
There is a need for more work on user aspects of BTs, including human robot interaction, but early examples include \cite{paxtonUserExperienceCoSTAR2017}.

\section{Reinforcement Learning, Utility and BTs} 
\label{sec_rl}

Reinforcement learning (RL) is a research area aiming to produce near optimal controllers for a very general family of control problems. Sometimes so-called end-to-end solutions can be found, mapping raw sensor readings to actions, to extremely challenging problems \cite{vinyalsGrandmasterLevelStarCraft2019}.
If all problems could be solved end-to-end using RL, there would be no need for BTs, but there is reason to believe that modular hierarchical control structures will still be useful for a number of years, especially since they enable safety guarantees, see Section \ref{sec_cbf}, and transparency to a human operator, see Section \ref{sec_exai}.
A natural question is then how we can combine RL with BTs, ideally to get the performance of RL, and the guarantees and transparency of BTs.

\begin{figure}[!h]
	\centering
	\includegraphics[width=10cm]{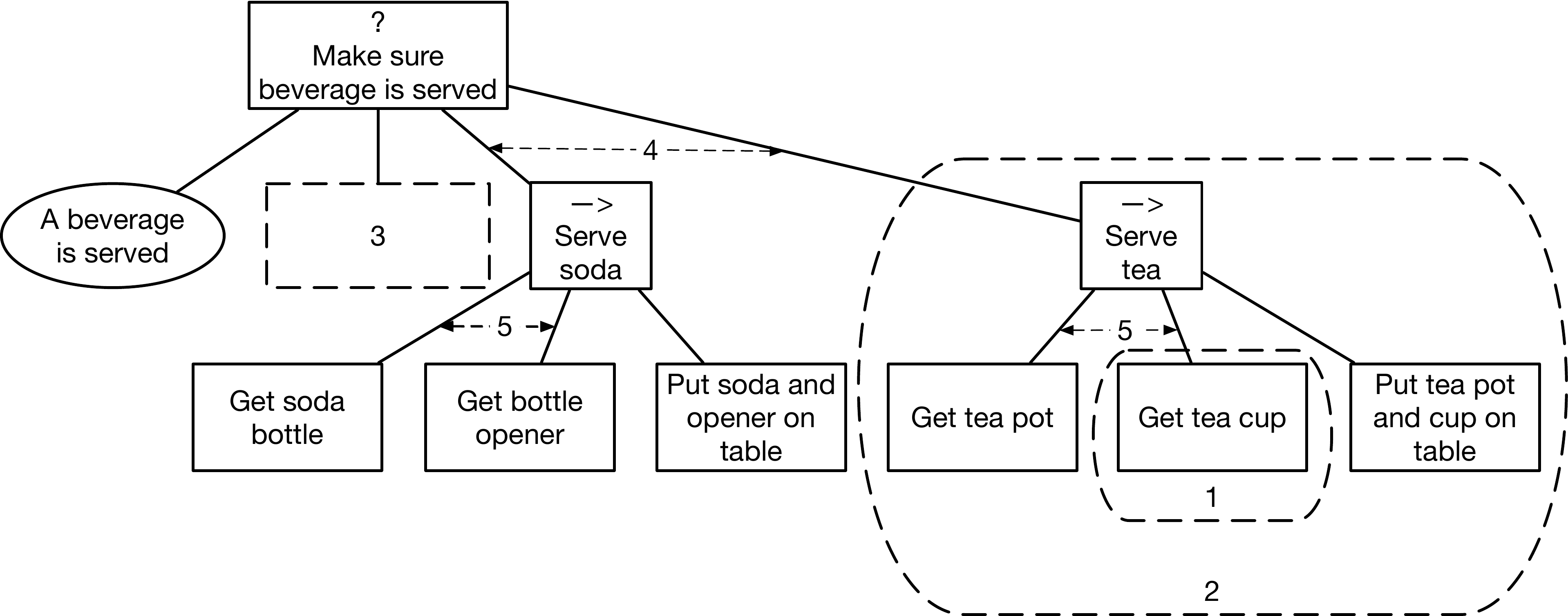}
	\caption{There are many ways of combining RL with BTs. Single actions can be replaced (1), as can entire subtrees (2). One could also add a new subtree (3), and keep the old subtrees as fallback options if the new one fails. Finally, the order of fallback options can be changed (4), as can the order of satisfying  preconditions (5).}
	\label{RL_options}
\end{figure} 

A number of different ways of combining BTs with RL are illustrated in Figure~\ref{RL_options}.
The first option that comes to mind is perhaps to replace a single action with RL (1 in Figure~\ref{RL_options}).
This was explored in \cite{pereira_framework_2015,kartasev_integrating_2019}, where an RL problem including states, actions and rewards was specified by the user.
If the problem domain is well suited for RL this approach can then be expanded, replacing subtrees (2 in Figure~\ref{RL_options}) in a bottom up fashion. This gradual approach can be seen as a low risk option to replacing an entire control structure with end-to-end RL.

RL can also be used to increase the performance of an existing BT. One way of doing this is to keep the current structure, but add an additional option using RL, as suggested in \cite{spragueAddingNeuralNetwork2018} and illustrated in (item 3 of Figure~\ref{RL_options}). If the RL option fails, the other ones will execute and achieve the subgoal. 

Another approach, (item 4 in Figure~\ref{RL_options}), focussing on subtree order was explored in \cite{dey_ql-bt_2013,fu_reinforcement_2016,zhang_combining_2017,zhu_behavior_2019}.
As the Q-value of a state-action pair in RL estimates the future reward, it was noted that Q-values could be used to choose between fallback options, reordering them based on the Q-value in the current state.
A similar idea was explored in \cite{hannaford_simulation_2016}, where the success probability of each child was estimated by gathering data during execution, and the order was updated to keep the child with highest value first.
Finally, the least explored option (item 5 in Figure~\ref{RL_options}) is to also reorder pre-conditions in the BT. For example, fetching a bunch of items, as in Figure~\ref{RL_options}, amounts to a small instance of a traveling salesman problem (TSP) where ordering might have impact on performance.

\section{Evolutionary Algorithms and BTs} 
\label{sec_ga}

Evolutionary algorithms, or genetic algorithms,
\cite{whitleyGeneticAlgorithmTutorial1994}, are local optimization algorithms inspired by the theory of evolution.
The basic idea is to maintain a family of solution candidates, and then create new solution candidates from the previous ones by applying mutations (small alteration of a candidate) and crossover (taking two candidates, mark a subset in each and swap the subsets). 
The candidates are then evaluated using a fitness function, and some portion of them are removed.

The modularity of BTs, with a uniform interface on all subtree levels, makes them well suited for evolutionary algorithms.
Applying mutation to a BT can be done by picking an arbitrary subtree, and replacing it with some other subtree, as illustrated in Figure~\ref{ga_mutation}(a). 
Futhermore, crossover can similarly be done by taking two BTs, choosing a subtree in each and swap them, as illustrated in Figure~\ref{ga_crossover}(b).

\begin{figure}[!h]
	\centering
	\begin{subfigure}[b]{0.45\textwidth}
         \centering
         \includegraphics[width=\textwidth]{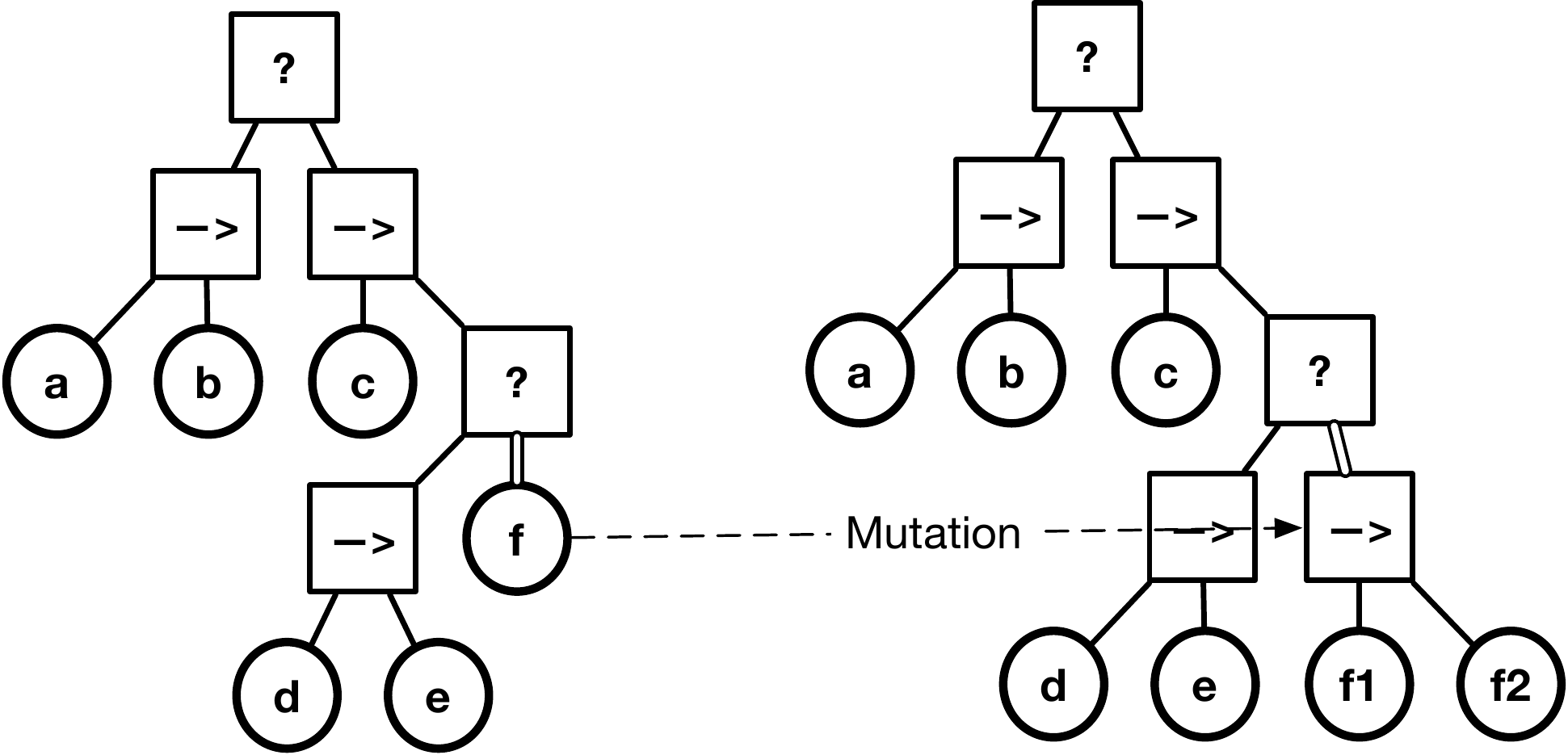}
         \caption{}
      \end{subfigure}
     \hfill
     \begin{subfigure}[b]{0.45\textwidth}
         \centering
         \includegraphics[width=\textwidth]{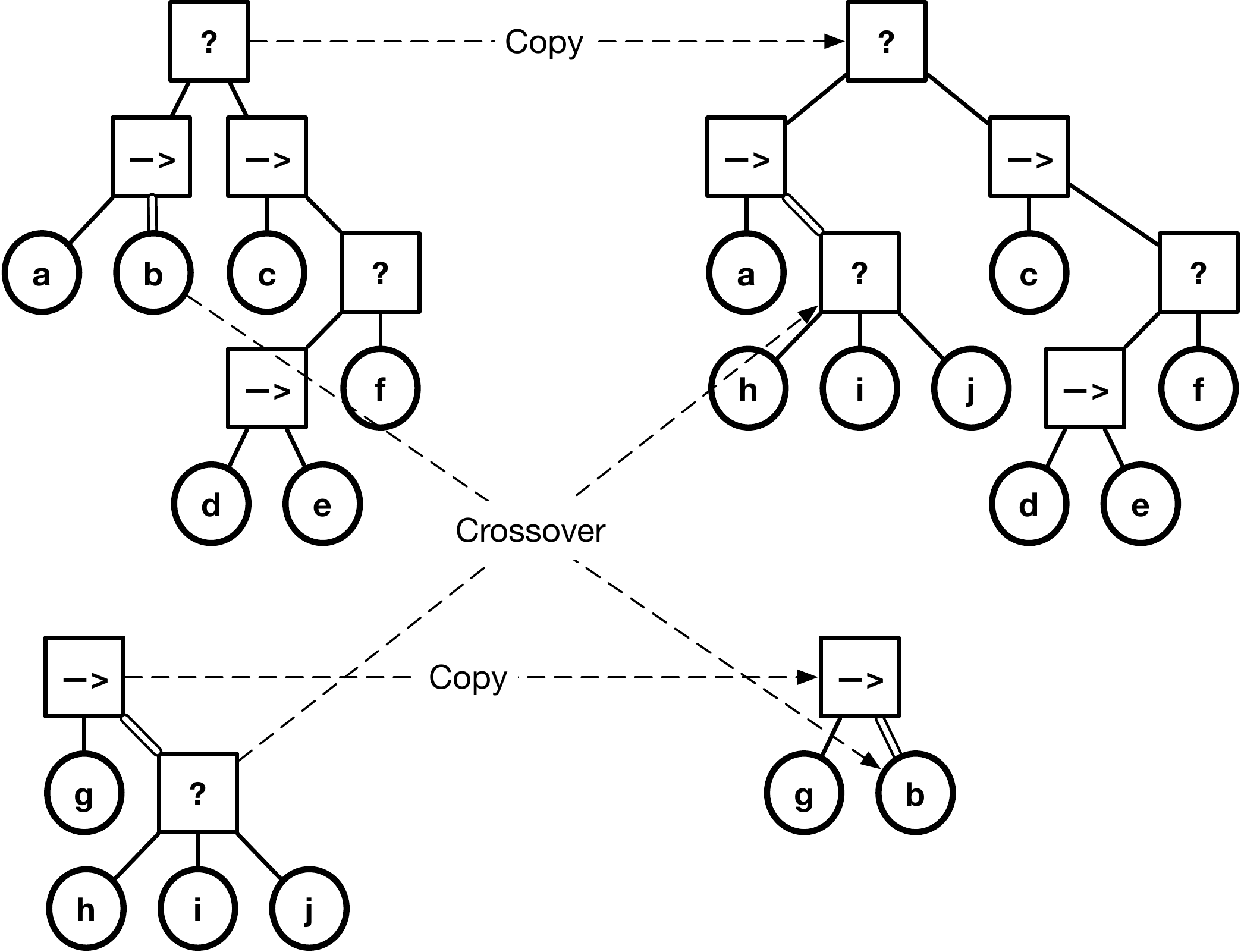}
                  \caption{}
     \end{subfigure}
     \vspace{5mm}
	\caption{Applying mutation and crossover to a BT }	
	\label{ga_mutation}
	\label{ga_crossover}
\end{figure}

%

It has been shown that \emph{locality}, in terms of small changes in design giving small changes in performance, is important for the performance of evolutionary algorithms \cite{rothlauf2006locality}.
As seen in Lemma \ref{lem_influence_region}
 above, the influence region $I_i$ captures the part of the state space where a subtree can influence system behavior.
 Thus, for larger BTs, a given subtree will often have a fairly limited $I_i$ providing the locality described in \cite{rothlauf2006locality}.

A well known problem of evolutionary algorithms is so-called bloating \cite{lukeComparisonBloatControl2006}, where the average size of the  individuals in a population grows, without a corresponding increase in fitness. For BTs it is clear that a design could have large parts that do not contribute at all.
Both \cite{colledanchise_learning_2019} and \cite{hallawa_evolving_2020} describe methods for addressing this problem.
A practical approach is to try pruning different subtrees to see if fitness is reduced, whereas a more theoretical approach is to compute the influence regions $I_i$ for all the subtrees, and remove the ones with $I_i=\emptyset$.
%
%

Work on BTs and evolutionary algorithms can be found in 
\cite{limEvolvingBehaviourTrees2010,nicolauEvolutionaryBehaviorTree2017,colledanchise_learning_2019,iovinoLearningBehaviorTrees2020,paduraru_automatic_2019}.
Finally, combinations of BTs, evolutionary algorithms and planning can be found in  
\cite{styrudCombiningPlanningLearning2021}, and multi-agent problems using BTs and evolutionary algorithms have been addressed in 
\cite{neupane_learning_2019,jones_evolving_2018}.

%

%

\section{Planning and BTs} 
\label{sec_planning}

Planning algorithms are typically used to create a sequence of actions that will move the world state from a given starting state to some desired goal state. Planning can either be on a lower level, such as motion planning or grasp planning, or a higher level, such as task planning.
Low level planning are usually integrated as leaves in a BT, whereas high level planning can be used to create the BT itself.
The reason for combining BTs and planning algorithms is often to add the reactive feedback properties of a BT to the goal directed actions output by the planner.

The most straightforward way of using a task planner is to first run the planner to get a sequence of actions, and then execute this sequence. This works fine if the world is static and the actions are predictable. However, if an action fails, the sensing of the world was inaccurate, or an external agent changes the world state, the planned sequence of actions will not lead to the goal state.
A natural way to add feedback to the system is to monitor the execution to see if it runs as predicted, and re-plan once there is a significant enough deviation.
However, in many cases the original plan is still valid, we just need to jump to the proper phase. If a grasping action fails, we can try grasping again. If an object is picked up and later dropped, we can jump back to the pick up action. If an external agent helps us with some subgoal, we can jump ahead to the proper action in the list. As seen above, BTs are an appropriate tool for supplying this kind of feedback control, and were used in  e.g. \cite{paxtonRepresentingRobotTask2019}.

Planning might also be used to create a feedback policy. One example of this is the A* algorithm that computes the shortest path to goal from all initial states, not only the one currently occupied. Thus, if an unexpected action not only moves us back or forth on the expected path, but also sideways to a state that was not intended to be occupied, the plan still contains the proper action. The design described in Section \ref{sec_design} provides this functionality in a BT  \cite{ogrenConvergenceAnalysisHybrid2020}. An advantage is the larger region of attraction, while a drawback is that the BT needed to cover all potential situations can be very large.

Another use of the reactivity of BTs in connection to planning can be found in 
\cite{martinOptimizedExecutionPDDL2021}.
Here a BT is created from the output of the planner with the intent of reactively taking advantage of opportunities for parallel execution, and tasks finishing earlier than expected. By tracking the preconditions for each action, they are executed as early as possible, based on information that was not available at planning time.

The combination of planners with BT has been explored in 
     \cite{colledanchiseBlendedReactivePlanning2019,tadewos_automatic_2019,tadewos_--fly_2019,zhou_autonomous_2019,paxton19,schwab_capturing_2015,kucklingBehaviorTreesControl2018,styrudCombiningPlanningLearning2021}.
     Furthermore, the special cases of HTN planners were investigated in 
  \cite{neufeld_hybrid_2018,rovida_extended_2017,segura-muros_integration_2017,holzl_reasoning_2015},  and LTL planners in  
  \cite{colledanchise_synthesis_2017,lan_autonomous_2019,biggarFrameworkFormalVerification2020}.

\section{Conclusions} 
\label{sec_conclusions}
In this paper we have provided a control theoretic approach to BTs, showing how they can be seen as a hierarchically modular way to create a switched dynamical system, where the switching is based on feedback from lower level modules. 
We have also showed how the resulting operating regions can be computed, based on the specifications of parents, siblings and children of the node. Using these operating regions, we present sufficient conditions of convergence to the goal region of the entire BTs, as well as practical designs that can be used to create convergent BTs. Finally, we have showed how these core result connect to other research efforts on BTs, including control barrier functions, explainable AI, reinforcement learning, genetic algorithms and planning.

\begin{summary}[SUMMARY POINTS]
\begin{enumerate}
\item Behavior trees represent a hierarchically modular way to combine controllers into more complex controllers.
\item Behavior trees enable feedback control, not only on the lowest level, but on all levels, as the interface explicitly includes meta information (feedback) regarding the applicability and progress of a controller, that enables the parent level to act based on this feedback.
\item The modular structure of behavior trees lends itself to formal analysis regarding convergence and region of attraction.
\item  Ongoing work connects behavior trees to other research areas such as planning and learning.
\end{enumerate}
\end{summary}

\begin{issues}[FUTURE ISSUES]
\begin{enumerate}
\item Reinforcement learning can solve many problems end-to-end. However, many robot systems will need a modular structure combining separate capabilities, such as path planning and grasping. Behavior trees is a viable option for this structure and the connections between reinforcement learning and behavior trees needs to be explored further.
\item Explainable AI, learning by demonstration and human robot interaction (HRI) are areas where the transparency of BTs could play an important role. 
\item Behavior trees have been explored from an AI and robotics perspective, but very little work has been done from a control theoretic point of view.
\end{enumerate}
\end{issues}

\section*{DISCLOSURE STATEMENT}
The authors are not aware of any affiliations, memberships, funding, or financial holdings that
might be perceived as affecting the objectivity of this review. 

\section*{ACKNOWLEDGMENTS}
The authors gratefully acknowledge the support from SSF  through  the  Swedish  Maritime Robotics Centre (SMaRC) (IRC15-0046), and by FOI through  project 7135.

%

\bibliographystyle{ar-style3}
\bibliography{biblioShort,non_bt_papers,MyZoteroLibrary}

\end{document}